\documentclass[12pt]{alt2022} % Anonymized submission 
\title[Faster Rates of DP-SCO]{Faster Rates of Private Stochastic Convex Optimization}
\usepackage{times}
\usepackage[utf8]{inputenc}
\usepackage[T1]{fontenc}    % use 8-bit T1 fonts
\usepackage{hyperref}       % hyperlinks
\usepackage{url}            % simple URL typesetting
\usepackage{booktabs}     
\usepackage{amsfonts}       % blackboard math symbols
\usepackage{nicefrac}       % compact symbols for 1/2, etc.
\usepackage{microtype}      % microtypography
\usepackage{amsfonts}       % blackboard math symbols
\usepackage{nicefrac}       % compact symbols for 1/2, etc.
\usepackage{enumerate}
\usepackage{microtype}      % microtypography
\usepackage{times}
\usepackage{adjustbox}
\usepackage{amsmath}
\usepackage{bbm}
\usepackage{blindtext}
\usepackage{enumitem}
\usepackage{graphicx}
\usepackage{color}
\usepackage{xcolor}  
\newtheorem{case}{Case}

\renewcommand{\tilde}{\widetilde}

\altauthor{%
 \Name{Jinyan Su} \Email{huizhang@std.uestc.edu.cn}\\
 \addr {School of Mathematical Sciences \\
 University of Electronic Science and Technology of China\footnote{Part of the work was done when Jinyan Su  was a research intern at KAUST.}}
 \AND
  \Name{Lijie Hu} \Email{lijie.hu@kaust.edu.sa}\\
 \addr {Division of Computer, Electrical and Mathematical Sciences and Engineering\\ King Abdullah University of Science and Technology}
 \AND
 \Name{Di Wang} \Email{di.wang@kaust.edu.sa}\\
\addr {Division of Computer, Electrical and Mathematical Sciences and Engineering\\ King Abdullah University of Science and Technology}
}

\begin{document}

\maketitle

\begin{abstract}
   In this paper, we revisit the problem of Differentially Private Stochastic Convex Optimization (DP-SCO) and provide excess population risks for some special classes of functions that are faster than the previous results of general convex and strongly convex functions. In the first part of the paper, we study the case where the population risk function satisfies the Tysbakov Noise Condition (TNC) with some parameter $\theta>1$. Specifically, we first show that under some mild assumptions on the loss functions, there is an algorithm whose output could achieve an upper bound of $\tilde{O}((\frac{1}{\sqrt{n}}+\frac{d}{n\epsilon})^\frac{\theta}{\theta-1}) $ and $\tilde{O}((\frac{1}{\sqrt{n}}+\frac{\sqrt{d\log(1/\delta)}}{n\epsilon})^\frac{\theta}{\theta-1})$ for $\epsilon$-DP and $(\epsilon, \delta)$-DP, respectively when $\theta\geq 2$, here $n$ is the sample size and $d$ is the dimension of the space. Then we address the inefficiency issue, improve the upper bounds by $\text{Poly}(\log n)$ factors and extend to the case where $\theta\geq \bar{\theta}>1$ for some known $\bar{\theta}$. Next we show that the excess population risk of population functions satisfying TNC with parameter $\theta\geq 2$ is always lower bounded by  $\Omega((\frac{d}{n\epsilon})^\frac{\theta}{\theta-1}) $ and $\Omega((\frac{\sqrt{d\log(1/\delta)}}{n\epsilon})^\frac{\theta}{\theta-1})$ for $\epsilon$-DP and $(\epsilon, \delta)$-DP, respectively, which matches our upper bounds. In the second part, we focus on a special case where the population risk function is strongly convex. Unlike the previous studies, here we assume the loss function is {\em non-negative} and {\em the optimal value of population risk is sufficiently small}. With these additional assumptions, we propose a new method whose output could achieve an upper bound of $O(\frac{d\log(1/\delta)}{n^2\epsilon^2}+\frac{1}{n^{\tau}})$ and $O(\frac{d^2}{n^2\epsilon^2}+\frac{1}{n^{\tau}})$ for any $\tau> 1$ in $(\epsilon,\delta)$-DP and $\epsilon$-DP model respectively if the sample size $n$ is sufficiently large. These results circumvent their corresponding lower bounds in \citep{feldman2020private} for general strongly convex functions.   Finally, we conduct experiments of our new methods on real world data. Experimental results also provide new insights into established theories.
\end{abstract}


\section{Introduction}
Preserving the privacy of training data has become an important consideration and now is a challenging task for machine learning algorithms. To address the privacy issue, Differential Privacy (DP) \citep{dwork2006calibrating}, which roots in cryptography, is a strong mathematical scheme for privacy preserving. It allows for rich statistical and machine learning analysis, and is now becoming a de facto notation for private data analysis.  Methods to guarantee differential privacy have been widely studied, and recently adopted in industry \citep{apple,ding2017collecting}.  

As one of the most important problems in Machine Learning and Differential Privacy community,  the Empirical Risk Minimization problem in the DP model, {\em i.e.,} DP-ERM, has been studied quite well in the last decade, starting from \citep{chaudhuri2011differentially},  such as \citep{bassily2014private, wang2017differentially,wang2019differentially,wu2017bolt,kasiviswanathan2016efficient,kifer2012private,smith2017interaction,wang2018empirical,wang2019noninteractive,asi2021private}. Besides DP-ERM, its population (or expected) version, namely Differentially Private Stochastic Convex Optimization (DP-SCO), has received much attention in recent years, starting from \citep{bassily2014private}. Specifically,  \citep{bassily2019private} first provides the optimal rate of DP-SCO with general convex loss functions in $(\epsilon, \delta)$-DP, which is quite different from the optimal rate in DP-ERM. Later, \citep{feldman2020private} extends this problem to strongly convex and (or) non-smooth cases by providing a general localization technique. Moreover, their methods have linear time complexity if the loss functions are smooth. For non-smooth loss functions, \citep{kulkarni2021private} recently proposes a new method which only need subquadratic gradient complexity. While there are already a large number of studies on DP-SCO, the problem is still far from well understood. A key observation is that, all of the previous 
works only focus on the the case where the loss functions are either general convex or strongly convex. However, there are also many problems that are even stronger than strongly convex functions, or fall between convex and strongly convex functions. In the non-private counterpart, various studies have attempted to get faster rates by imposing additional assumptions on the loss functions. And it has been shown that it is indeed possible to achieve rates that are faster than the rates of general convex loss functions \citep{yang2018simple,koren2015fast,van2015fast}, or it could even achieve the same rate as in the strongly convex case even if the function is not strongly convex \citep{karimi2016linear,liu2018fast,xu2017stochastic}. Motivated by this, our question is,

{\bf For the problem of DP-SCO with  special classes of population risk functions, is it possible to achieve faster rates of the excess population risk than the optimal ones of general convex and (or) strongly convex cases? }

In this paper, we provide an affirmative answer by studying some classes of population risk functions. Particularly, we will mainly focus on the case where the population risk satisfies the Tysbakov Noise Condition (TNC) \footnote{In some related work it is also called the Error Bound Condition  or the Growth Condition \citep{liu2018fast,xu2017stochastic}.}, which includes strongly convex functions, SVM and linear regression as special cases. Our contributions can be summarized as follows. 
\begin{itemize}
    \item In the first part of the paper, we study the problem where the population risk satisfying TNC with parameter $\theta$ and propose three methods. When $\theta\geq 2$, we first propose a method that could achieve an excess population risk of $\tilde{O}((\frac{1}{\sqrt{n}}+\frac{d}{n\epsilon})^\frac{\theta}{\theta-1}) $ and $\tilde{O}((\frac{1}{\sqrt{n}}+\frac{\sqrt{d\log(1/\delta)}}{n\epsilon})^\frac{\theta}{\theta-1})$ in $\epsilon$-DP and $(\epsilon, \delta)$-DP model respectively under the assumption that the loss function is smooth and Lipschitz,  where $n$ is the sample size of the data and $d$ is the dimension of the space. We then propose another method to resolve the inefficiency issue under the assumption that $\theta$ is known. Moreover, we propose an improved method. Compared with previous two methods, it improves the upper bounds of error by $\text{Poly}(\log n)$ factors. And it only needs a relaxed assumption of $\theta\geq \bar{\theta}>1$ for some known $\bar{\theta}$ instead of $\theta$ being known or $\theta\geq 2$. Moreover, it outperforms the previous methods practically. Next, we focus on the lower bounds of the excess population risk. Specifically, for any $\theta\geq 2$, we show that there is a population risk function satisfying TNC with parameter $\theta$ such that for any $\epsilon$-DP ($(\epsilon, \delta)$-DP) algorithm, its output achieves an excess risk of  $\Omega((\frac{d}{n\epsilon})^\frac{\theta}{\theta-1}) $ ( $\Omega((\frac{\sqrt{d\log(1/\delta)}}{n\epsilon})^\frac{\theta}{\theta-1})$) with high probability.
    \item In the second part of the paper, we will focus on the problem where the population risk function is strongly convex, which is a special case of TNC functions with $\theta=2$. Unlike the previous studies, here we assume the loss function is non-negative and the optimal value of the population is sufficiently small. With these additional assumptions, we propose a new method whose output could achieve an upper bound of $O(\frac{d\log(1/\delta)}{n^2\epsilon^2}+\frac{1}{n^{\tau}})$ and $O(\frac{d^2}{n^2\epsilon^2}+\frac{1}{n^{\tau}})$ for any $\tau>1$ in $(\epsilon,\delta)$-DP and $\epsilon$-DP model respectively if the sample size $n$ is sufficiently large. These rates circumvent their corresponding lower bounds for general strong convex functions in \citep{feldman2020private}, {\em i.e.,} $\Theta(\frac{d^2}{n^2\epsilon^2}+\frac{1}{n})$ for $\epsilon$-DP and $\Theta(\frac{d\log(1/\delta)}{n^2\epsilon^2}+\frac{1}{n})$ for $(\epsilon, \delta)$-DP.
\end{itemize}
Due to the space limit, all the experiments (Appendix \ref{sec:experiments}) and proofs are included in Appendix. 
\section{Related Work}
Starting from \citep{chaudhuri2011differentially}, a long list of works have attacked the problems of DP-ERM
%on them spanning over last decade which attack the problems
from different perspectives: \citep{bassily2014private, iyengar2019towards,zhou2020bypassing,song2020characterizing,wang2017differentially,zhang2017efficient} studied the problems in the low dimensional case and the central model, \citep{kasiviswanathan2016efficient,kifer2012private,talwar2015nearly,wang2020knowledge,cai2020cost} considered the problems in the high dimensional sparse case and the central model, \citep{smith2017interaction,duchi2013local,JMLR:v21:19-253,duchi2018minimax} focused on the problems in the local model. However, almost all of these works only focus the case where the 
empirical risk function is either general convex or strongly convex. For special class of functions, \citep{wang2017differentially} studies the empirical risk functions satisfying Polyak-{L}ojasiewicz (PL) condition, which is weaker than strongly convexity and show that it is possible to achieve an excess empirical risk of $O(\frac{d\log(1/\delta)}{n^2\epsilon^2})$, which is the same as the strongly convex loss. As we will mention in Remark \ref{remark:3}, the PL condition is equivalent to TNC with parameter $\theta=2$. Thus, in this paper we extend the result from the empirical risk to the population risk function.

For DP-SCO, besides the related work we mentioned in the previous section, there is another direction which studies some special cases of DP-SCO. For example, \citep{bassily2021non} and \citep{asi2021private} consider the case where the underlying constraint set $\mathcal{W}$ has specific geometric structures, such as polyhedron. \citep{guzman2021differentially} studies the (non)smooth and (non)convex generalized linear loss. \citep{wang2020differentially} and \citep{kamath2021improved} focus on the case where the distribution of the data or the gradient of the loss function is heavy-tailed. However, none of these works study the case where the population risk satisfies TNC.  \citep{liu2021revisiting} recently studies the theoretical guarantees of the PATE model \citep{papernot2016semi} under the assumption that the population risk function satisfies TNC and shows that it is possible to achieve faster rates than in the convex case 
\citep{bassily2018model}. However, since here we focus on a different problem, their results cannot be used to DP-SCO. 

\paragraph{Concurrent Work:} We notice that \citep{asi2021adapting} also studies DP-SCO with TNC population risk functions concurrently. However, compared with its results there are several critical differences. 

\textbf{1)}  The idea of Algorithm 2 in \citep{asi2021adapting} is similar to Algorithm \ref{alg:3} in our paper. However, the idea of proof and the choice of parameters are quite different. 

\textbf{2)} The same as Algorithm \ref{alg:3}, Algorithm 2 in \citep{asi2021adapting}  is also inefficient and has poor performance in practice. To resolve the issue, we also develop two other  algorithms (Algorithm \ref{alg:new} and \ref{alg:2}). 

\textbf{3)}  For $(\epsilon, \delta)$-DP model, \citep{asi2021adapting} only shows the worst-case lower bound of  $\Omega((\frac{\sqrt{d\log(1/\delta)}}{n\epsilon})^\frac{\theta}{\theta-1})$ under the assumption $\theta\geq 1+c$ for some constant $c>0$ while in this paper we also extend the result to $\theta\geq 2$. Although the hard instance in \citep{asi2021adapting} is similar to ours, the proofs of lower bounds are different. 

\textbf{4)}  In this paper, we also provide experimental results on the problem which has not been studied in \citep{asi2021adapting}. 

\textbf{5)} Besides TNC population risk functions, in this paper we also provide faster rates of DP-SCO with strongly convex loss function with additional assumptions which also has not been studied in \citep{asi2021adapting}. 
%\vspace{-0.1in}
	\section{Preliminaries}
	%\vspace{-0.1in}
	\begin{definition}[Differential Privacy \citep{dwork2006calibrating}]\label{def:3.1}
	Given a data universe $\mathcal{X}$, we say that two datasets $S,S'\subseteq \mathcal{X}$ are neighbors if they differ by only one entry, which is denoted as $S \sim S'$. A randomized algorithm $\mathcal{A}$ is $(\epsilon,\delta)$-differentially private (DP) if for all neighboring datasets $S,S'$ and for all events $E$ in the output space of $\mathcal{A}$, the following holds
	$\text{Pr}(\mathcal{A}(S)\in E)\leq e^{\epsilon} \text{Pr}(\mathcal{A}(S')\in E)+\delta.$
\end{definition}
%In this paper, we will focus on both $\epsilon$ and $(\epsilon, \delta)$-DP and  we will mainly use Gaussian mechanism and Laplacian mechanism to guarantee the DP property. 
	\begin{definition}[Gaussian Mechanism]
	Given any function $q : \mathcal{X}^n\rightarrow \mathbb{R}^d$, the Gaussian mechanism is defined as  $q(S)+\xi$ where $\xi\sim \mathcal{N}(0,\frac{16\Delta^2_2(q)\log(1/\delta)}{\epsilon^2}\mathbb{I}_d)$, \footnote{For simplicity to theoretical analysis, throughout the paper we use constant 16 for Gaussian mechanism. In practice we can use smaller constants.} where where $\Delta_2(q)$ is the $\ell_2$-sensitivity of the function $q$,
{\em i.e.,}
		$\Delta_2(q)=\sup_{S\sim S'}||q(S)-q(S')||_2.$	Gaussian mechanism preserves $(\epsilon, \delta)$-DP for $0<\epsilon, \delta\leq 1$.
	\end{definition}
\begin{definition}[Laplacian Mechanism]\label{def:3.3}
	{ Given any function $q : \mathcal{X}^n\rightarrow \mathbb{R}^d$, the Laplacian mechanism is defined as $\mathcal{M}_G(S,q,\epsilon)=q(S)+ (Y_1, Y_2, \cdots, Y_d),$
	where each $Y_i$ is i.i.d. drawn from a Laplacian Distribution $\text{Lap}(\frac{\Delta_1(q)}{\epsilon}),$ where $\Delta_1(q)$ is the $\ell_1$-sensitivity of the function $q$, {\em i.e.,}
		$\Delta_1(q)=\sup_{S\sim S'}||q(S)-q(S')||_1.$ For a parameter $\lambda$, the Laplacian distribution has the density function: 
$\text{Lap}(x|\lambda)=\frac{1}{2\lambda}\exp(-\frac{x}{\lambda}).$
		Laplacian Mechanism preserves $\epsilon$-DP.}
\end{definition}

	\begin{definition}[DP-SCO \citep{bassily2014private}]\label{definition:1}
		Given a dataset $S=\{x_1,\cdots,x_n\}$ from a data universe $\mathcal{X}$ where $x_i$ are i.i.d. samples from some unknown distribution $\mathcal{D}$, a convex loss function $f(\cdot, \cdot)$, and a convex constraint set  $\mathcal{W} \subseteq \mathbb{R}^d$, Differentially Private Stochastic Convex Optimization (DP-SCO) is to find $w^{\text{priv}}$ so as to minimize the population risk, {\em i.e.,} $F (w)=\mathbb{E}_{x\sim \mathcal{D}}[f(w, x)]$
		with the guarantee of being differentially private.
		 The utility of the algorithm is measured by the \textit{(expected) excess population risk}, that is  $\mathbb{E}_{\mathcal{A}}[F (w^{\text{priv}})]-\min_{w\in \mathbb{\mathcal{W}}}F(w),$
where the expectation of $\mathcal{A}$ is taken over all the randomness of the algorithm. Besides 
%Instead of 
the population risk, we can also measure the \textit{empirical risk} of dataset $S$: $\bar{F}(w, S)=\frac{1}{n}\sum_{i=1}^n f(w, x_i).$
	\end{definition}

%In the following we will introduce some properties of loss functions, which are commonly used in SCO related work. 
	\begin{definition}\label{def:5}
		A function $f:\mathcal{W}\mapsto \mathbb{R}$ is L-Lipschitz over the domain $\mathcal{W}$ if for all $w,w^{\prime}\in \mathcal{W}$,
		$|f(w)-f(w^{\prime})|\leq L||w-w^{\prime}||_2.$
	\end{definition}
	\begin{definition}\label{def:6}
		A function $f:\mathcal{W}\mapsto \mathbb{R}$ is $\beta$-smooth over the domain $\mathcal{W}$ if for all $w,w' \in \mathcal{W}$, 
		$f(w)\leq f(w')+\langle \nabla f(w'), w-w'\rangle +\frac{\beta}{2}||w-w'||_2^2.$
	\end{definition}
	\begin{definition}\label{def:7}
		A function $F:\mathcal{W}\mapsto \mathbb{R}$ is $\lambda$-strongly convex over the domain $\mathcal{W}$ if, for all $w,w^{\prime}\in \mathcal{W}$,$
		F(w)+\langle\nabla F(w),w^{\prime}-w\rangle+\frac{\lambda}{2}||w^{\prime}-w||_2^2\leq F(w^{\prime}). 
		$
	\end{definition}
	Let $w^*=\arg \min_{w\in \mathcal{W}} F(w)$ be the minimizer, 
	strongly convexity implies \citep{hazan2011beyond}:
	\begin{equation}\label{eq:1}
	F(w)-F(w^*)\geq \frac{\lambda}{2}||w-w^*||_2^2, \forall w\in \mathcal{W}.
	\end{equation}
     Previous work on DP-SCO only focused on case where the loss function is  either  convex or strongly convex \citep{bassily2019private,feldman2020private}. In this paper, we will mainly study the case where the population risk satisfies the Tysbakov Noise Condition (TNC) \citep{ramdas2012optimal,liu2018fast}, which has been studied quite well and has been shown that it could achieve faster rates than the optimal one of general convex loss functions in the non-private case.  Below we provide the definition of TNC. 
     \begin{definition}\label{def:8}
         For  a convex function $F(\cdot)$, let $\mathcal{W}_*=\arg\min_{w\in \mathcal{W}} F(w)$ denote the optimal set and for any $w\in \mathcal{W}$, let $w^*=\arg\min_{u\in \mathcal{W}_*}\|u-w\|_2$ denote the projection of $w$ onto the optimal set $\mathcal{W}_*$. 
        Function $F$ satisfies $(\theta, \lambda)$-TNC for some $\theta>1 $ and $\lambda>0$ if for any $w\in \mathcal{W}$ the following inequality holds
		\begin{equation}\label{eq:2}
		F(w)- F(w^*)\geq \lambda||w-w^*||_2^{\theta}. 
		\end{equation}
     \end{definition}
    From the definition of TNC and (\ref{eq:1}) we can see that for a $\lambda$-strong convex function it is $(2, \frac{\lambda}{2})$-TNC. Moreover, if a function is $(\theta, \lambda)$-TNC, then it is also $(\theta', \lambda)$-TNC for any $\theta<\theta' 
    $. Throughout the whole paper we will assume that $\theta$ is a constant and thus we will omit the term of $c^{\theta}$ in the Big-$O$ notation if $c$ is a constant. 
    
    \iffalse
    The following example is a population risk function that satisfies TNC. 
 \paragraph{Example of TNC \citep{liu2018fast}} Consider the quadratic problem with $\ell_p$-norm regularization
 \begin{equation*}
     \min_{w\in \mathcal{W}} F(w)=   
       w^T\mathbb{E}_x[A(x)]w+w^T\mathbb{E}_x[b(x)]+c+\lambda \|w\|_p^p,
     \end{equation*}
         where $c$ is a constant, $\mathbb{E}_x[A(x)]$ is a positive semi-definite matrix, $\mathcal{W}$ is a bounded polyhedron and $p\geq 2$. Then $F(\cdot)$ is TNC with $\theta=p$. 
 \fi 
%For a TNC function, we have the following property. 
	\begin{lemma}[Lemma 2 in \citep{ramdas2012optimal}]\label{lemma:3}
	If the function $F(\cdot)$ is $(\theta, \lambda)$-TNC and  $L$-Lipschitz, then we have $||w-w^*||_2\leq(L\lambda^{-1})^{\frac{1}{\theta-1}}$  and $F(w)-F(w^*)\leq (L^{\theta}\lambda^{-1})^{\frac{1}{\theta-1}}$ for all $ w\in\mathcal{W}$, where $w^*$ is defined as in Definition \ref{def:8}.
	\end{lemma}
	%\vspace{-0.2in}
	\section{Optimal Rates of Excess Population Risk}
	%\vspace{-0.1in}
	\subsection{Upper Bounds of Excess Population Risk}
	In this section, we will concentrate on the case where the population risk function is $(\theta, \lambda)$-TNC and provide some upper bounds of its excess population risk. To provide a clear intuition of our methods, we will first assume that the loss functions are smooth. Later we will extend to the non-smooth case. 
	
		\begin{algorithm}
	\caption{Phased-SGD($w_0, \eta,n,\mathcal{W}$) algorithm \citep{feldman2020private} \label{alg:1}}
		\KwIn{Dataset $S=\{x_1,\cdots,x_n\} $, convex function $f:\mathcal{W}\times \mathcal{X}\mapsto \mathbb{R}$, initial point $w_0\in\mathcal{W}$, step size $\eta$ (will be specified later), privacy parameter $\epsilon$ and (or) $\delta$.}\
			Set $k= \lceil \log_2 n\rceil$. Partite the whole dataset $S$ into $k$ subsets $\{S_1,\cdots,S_k\}$. Denote $n_i$ as the number of samples in $S_i$, {\em i.e.,} $|S_i|=n_i$, where $n_i=\lfloor 2^{-i}n \rfloor$. \;
			
		\For {$i=1,\cdots ,k$}{
		Let $\eta_i=4^{-i}\eta$, $w_i^1=w_{i-1}$. \;
		
		\For{$t= 1,\cdots,n_i$}{
		Update $w_i^{t+1}=\prod_{\mathcal{W}}(w_i^{t}-\eta_i\nabla_{w}f(w_{i}^{t},x_i^t))$, where $x_i^t$ is the $t$-th sample of the set $S_i$.\;
		}
		Set $\overline{w}_i=\frac{1}{n_i+1}\sum \limits_{t=1}^{n_i+1} w_i^t$. \;
		
		For $(\epsilon,\delta)$-DP, $w_i=\overline{w}_i+\xi_i$, where $\xi_i \sim \mathcal{N}(0,\sigma_i^2\mathbb{I}_d)$ with $\sigma_i=\frac{4L\eta_i\sqrt{\log (1/\delta)}}{\epsilon}$. \;
		
		For $\epsilon$-DP, $w_i=\overline{w}_i+\xi_i$, where $\xi_i=(\zeta_1, \cdots, \zeta_d)$ with each $\zeta_j \sim \text{Lap}(\lambda)$ and $\lambda=\frac{4L\eta_i \sqrt{d}}{ \epsilon}$. \;
	}
	
		\Return {$w_k$}
\end{algorithm}
	We first consider an easier case, where the TNC parameter $\theta$ satisfies $\theta\geq 2$. Our algorithm is based on the localization technique proposed by \citep{feldman2020private}, which provides an algorithm, namely Phased-SGD (Algorithm \ref{alg:1}) for DP-SCO with general convex loss functions and shows that the algorithm could achieve the optimal rate of  excess population risk. 

		\begin{lemma}\label{lemma:4}[Theorem 4.4 in \citep{feldman2020private}]
		Let $\mathcal{W}\subseteq \mathbb{R}^d$ be a closed convex set and $f(\cdot, x)$ be $\beta$-smooth, convex and $L$-Lipschitz function over $\mathcal{W}$ for each $x$. If we set $\eta=\frac{D}{L}\min\{\frac{4}{\sqrt{n}}, \frac{\epsilon}{2\sqrt{d\log(1/\delta)}}\}$ and if $\eta\leq \frac{1}{\beta}$ ({\em i.e.,} $n$ is sufficiently large), then Algorithm \ref{alg:1} will be $(\epsilon,\delta)$-DP for all $\epsilon\leq 2\log (1/\delta)$. The output satisfies
	$\mathbb{E}[F(w_k)]-\min_{w\in \mathcal{W}}F(w) \leq 10LD\left(\frac{1}{\sqrt{n}}+\frac{\sqrt{d\log(1/\delta)}}{\epsilon n}\right).$
	Set $\eta=\frac{D}{L}\min\{\frac{4}{\sqrt{n}}, \frac{\epsilon}{d}\}$ and if $\eta\leq \frac{1}{\beta}$, then Algorithm \ref{alg:1} will be $\epsilon$-DP. Moreover, the output satisfies $
	\mathbb{E}[F(w_k)]-\min_{w\in \mathcal{W}}F(w) \leq 10LD\left(\frac{1}{\sqrt{n}}+\frac{d}{n\epsilon}\right), 
	$
	where  $D>0$ satisfies that $||w_0-w^*||_2\leq D$. 
	\end{lemma} 
	We propose our adaptive stochastic approximation algorithm, which is presented in Algorithm \ref{alg:3}. The updates are divided into $m$ stages. At each stage, the Phased-SGD algorithm is applied with $n_0$ samples. Each employment of the Phased-SGD algorithm is warm-started by the initial point that is returned from the last stage. 
		\begin{algorithm}
	\caption{Private Stochastic Approximation($w_1,n,R_0$) \label{alg:3}}
		\KwIn{ 	Dataset $S=\{x_1,\cdots,x_n\} $, convex function $f:\mathcal{W}\times \mathcal{X}\mapsto \mathbb{R}$, initial point $w_0\in\mathcal{W}$, privacy parameter $\epsilon$ and (or) $\delta$. }
	   	Set $\hat{w}_0=w_1, m=\lfloor \frac{1}{2}\log_2\frac{2n}{\log_2 n}\rfloor -1, n_0=\lfloor\frac{n}{m}\rfloor$. Partite the data $S$ into $m$ disjoint subsets $\{S_1, \cdots, S_m\}$ with each $S_i$ containing $n_0$ samples. \;
	    
	\For{$k=\{1,\cdots,m\}$}{
	 For $(\epsilon,\delta)$-DP, set $\gamma_k=\frac{R_{k-1}}{L}\cdot \min\left\{\frac{4}{\sqrt{n_0}},\frac{\epsilon}{2\sqrt{d\log(1/\delta)}}\right\}$ and $R_k=\frac{R_{k-1}}{2}$.\; 
	 
	 For $\epsilon$-DP, set $\gamma_k=\frac{R_{k-1}}{L}\cdot \min\left\{\frac{4}{\sqrt{n_0}},\frac{\epsilon}{d}\right\}$ and $R_k=\frac{R_{k-1}}{2}$. \;
	
	Denote	$\hat{w}_k=	\text{Phased-SGD}(\hat{w}_{k-1},\gamma_k,n_0,\mathcal{W}\cap \mathbb{B}(\hat{w}_{k-1},R_{k-1})),$ where $\mathbb{B}(\hat{w}_{k-1},R_{k-1})$ is a ball with center $\hat{w}_{k-1}$ and radius $R_{k-1}$.
	The Phased-SGD runs on the $k$-th subset $S_i$. \;
	}
		\Return {$\hat{w}_m$}
\end{algorithm}

%The following theorem states that, the output of Algorithm \ref{alg:3} achieve a nearly asymptotically the same excess population risk as in Theorem \ref{thm:1}, up to $\text{Poly}(\log n)$ terms. 

The following theorem states that,  the output of Algorithm \ref{alg:3} achieve an excess population risk of $\tilde{O}((\frac{1}{\sqrt{n}}+\frac{d}{n\epsilon})^\frac{\theta}{\theta-1}) $ and $\tilde{O}((\frac{1}{\sqrt{n}}+\frac{\sqrt{d\log(1/\delta)}}{n\epsilon})^\frac{\theta}{\theta-1})$ for $\epsilon$-DP and $(\epsilon, \delta)$-DP, respectively, if the population risk function satisfies TNC with $\theta\geq 2$.  
\begin{theorem}\label{thm:2}
Assume that $F(\cdot)$ satisfies $(\theta, \lambda)$-TNC and $f(\cdot, x)$ is  convex, $\beta$-smooth and $L$-Lipschitz for each $x$. Then Algorithm \ref{alg:3} is $\epsilon$-DP or $(\epsilon,\delta)$-DP based on different stepsizes $\{\gamma_k\}_{k=1}^m$  and noises if $\gamma_k\leq \frac{1}{\beta}$. Moreover, if $n$ is sufficiently large such that $n\geq256$, for $(\epsilon,\delta)$-DP we have 
		\begin{equation*}
\mathbb{E}[F(\hat{w}_m)]-\min_{w\in \mathcal{W}}F(w)= O\left(\left(\frac{L^\theta}{\lambda}\right)^{\frac{1}{\theta-1}}\cdot\left(\frac{\sqrt{\log n}}{\sqrt{n}}+\frac{\sqrt{d\log(1/\delta)}\log n}{n\epsilon}\right)^{\frac{\theta}{\theta-1}}\right).
\end{equation*}
And for $\epsilon$-DP we have 
$\mathbb{E}[F(\hat{w}_m)]-\min_{w\in \mathcal{W}}F(w)= O\left(\left(\frac{L^\theta}{\lambda}\right)^{\frac{1}{\theta-1}}\cdot\left(\frac{\sqrt{\log n}}{\sqrt{n}}+\frac{d\log n}{n\epsilon}\right)^{\frac{\theta}{\theta-1}}\right).$
\end{theorem} 	
In practice, the main difficulty on implementing Algorithm \ref{alg:3} is the projection onto the ball $\mathcal{W}\cap \mathbb{B}(\hat{w}_{k-1},R_{k-1})$ in each iteration of the Phased-SGD in each phase. In practice, this could be solved by using the Dykstra’s algorithm \citep{dykstra1983algorithm,boyle1986method}, which studied the \emph{best approximation problem:} given $m$ closed and convex sets $W_1, \cdots, \mathcal{W}_m \subseteq \mathbb{R}^d$ and a point $y\in \mathbb{R}^d$, we seek the point in $\mathcal{W}_1\bigcap \cdots \bigcap \mathcal{W}_m$ (assumed nonempty) closest to $y$, and solve
 $\min_{u\in \mathcal{W}_1\bigcap \cdots \bigcap \mathcal{W}_m} \|u-y\|_2.$ 
However, in theory, the theoretical guarantee of Theorem \ref{thm:2} may not be held if we use the  Dykstra’s algorithm under the privacy constraint. The main reason is that, Dykstra’s algorithm can only provide an approximate solution of the projection step. However, the approximate solution may not have the same $\ell_2$ (or $\ell_1$)-norm sensitivity as the exact solution. Thus, from this view, Algorithm \ref{alg:3} lacks of efficiency. 
\begin{algorithm}
	\caption{Phased-SGD-SC($w_0, \gamma, \epsilon, \delta$) \label{alg:new}}
		\KwIn{	Dataset $S=\{x_1,\cdots,x_n\} $, convex function $f:\mathcal{W}\times \mathcal{X}\mapsto \mathbb{R}$, initial point $w_0\in\mathcal{W}$, privacy parameter $\epsilon$ and (or) $\delta$. $D$ is a constant satisfying $D\geq \|w_0-w^*\|_2$.}
	    Partite the data $S$ into $k$ disjoint subsets $\{S_1, \cdots, S_k\}$, where $k=\lceil \log\log n\rceil$ and for each $i\in [k]$, $|S_i|=n_i=\lfloor 2^{i-2} n/\log n \rfloor$.\;
	    
		\For {$t=1, \cdots, k$}{
		 Let $w_{t}=\text{Phased-SGD}(w_{t-1}, \eta_t, n_t, \mathcal{W})$, where the Phased-SGD runs on the $t$-th subset $S_i$ with loss function $f(w, x)+\frac{1}{2\gamma}\|w-w_{0}\|_2^2$. For $(\epsilon,\delta)$-DP, $\eta_t=\frac{D}{L}\min\{\frac{4}{\sqrt{n_t}}, \frac{\epsilon}{2\sqrt{d\log(1/\delta)}}\}$. For $\epsilon$-DP,  $\eta_t=\frac{D}{L}\min\{\frac{4}{\sqrt{n_t}}, \frac{\epsilon}{d}\}$. \;
		}
		\Return{$w_k$}
\end{algorithm}
Instead of using the Dykstra’s algorithm, motivated by \citep{xu2017stochastic}, in the following, we present a new algorithm which only needs the projection onto $\mathcal{W}$. Briefly speaking, instead of considering the original stochastic function, we focus on the problem with an additional strongly convex regularization, {\em i.e.,}
   $ \min_{w\in \mathcal{W}}F(w)+\frac{1}{2\gamma}\|w-w_1\|_2^2$, 
where $w_1\in \mathcal{W}$ is some reference point and $\gamma$ is some parameter. Specifically, the same as in Algorithm \ref{alg:3}, we first divide the whole algorithm into $m$ stages. In each stage we hope to find a private estimator $w^k$ such that $w^k \approx \arg\min_{w\in \mathcal{W}}F(w)+\frac{1}{2\gamma_k}\|w-w_{k-1}\|_2^2$ with $\gamma_k$ changing with $k$. Specifically, we use Algorithm \ref{alg:new} to get such private estimator. Note that due to the additional $\ell_2$ regularization, now the function is strongly convex. Thus, instead of using the original Phased-SGD (Algorithm \ref{alg:1}) for general convex loss, here we use a strongly convex version of Phased-SGD, which is adopted from \citep{feldman2020private}. Moreover, 
since now we have an additional $\ell_2$-norm regularization, here we do not need the projection onto the balls  $\mathcal{W}\cap \mathbb{B}(\hat{w}_{k-1},R_{k-1})$ during updates compared with Algorithm \ref{alg:3}. 

		\begin{algorithm}
	\caption{Private Stochastic Approximation-II($w_0,n,R_0$) \label{alg:new1}}
		\KwIn{	 	Dataset $S=\{x_1,\cdots,x_n\} $, convex function $f:\mathcal{W}\times \mathcal{X}\mapsto \mathbb{R}$, initial point $w_0\in\mathcal{W}$, privacy parameter $\epsilon$ and (or) $\delta$. $\chi_0$ is a constant such that $\chi_0\geq F(w_0)-\min_{w\in \mathcal{W}}F(w)$. }
	   For $(\epsilon, \delta)$-DP, set $ m=\lfloor -\frac{\theta}{2(\theta-1)}\log_2 (\frac{L^2}{\lambda^\frac{2}{\theta}}(\frac{1}{n}+\frac{d\log(1/\delta)}{n^2\epsilon^2}))   \rfloor, n_0=\lfloor\frac{n}{m}\rfloor$. For $\epsilon$-DP, set $ m=\lfloor -\frac{\theta}{2(\theta-1)}\log_2 (\frac{L^2}{\lambda^\frac{2}{\theta}}(\frac{1}{n}+\frac{d^2}{n^2\epsilon^2}))   \rfloor, n_0=\lfloor\frac{n}{m}\rfloor$. Partite the dataset $S$ into $m$ disjoint subsets $\{S_1, \cdots, S_m\}$ with each $S_i$ containing $n_0$ samples. \;
	   
	   Set $
	    {\gamma_0}= \chi_0 / ({6400 L^2}(\frac{1}{n_0}+\frac{d\log(1/\delta)}{n_0^2\epsilon^2}))$ for $(\epsilon, \delta)$-DP and  $\gamma_0 = \chi_0/ ({6400 L^2}(\frac{1}{n_0}+\frac{d^2}{n_0^2\epsilon^2}))$ for $\epsilon$-DP. \;
	    
	\For{$k=1,\cdots,m$}{
    Set $\gamma_k= \frac{\gamma_{k-1}}{2}$. \;

	Denote	$w_k=	\text{Phased-SGD-SC}(w_{k-1},\gamma_k, \epsilon, \delta)$.\;
	}
		\Return {$w_m$}
\end{algorithm}
	\begin{theorem}\label{new:th1}
Assume that $F(\cdot)$ satisfies $(\theta, \lambda)$-TNC and $f(\cdot, x)$ is  convex, $\beta$-smooth and $L$-Lipschitz for each $x$. If $n$ is sufficiently large such that $\bar{\theta}\geq 2^{\frac{\log\log n}{\log n}}$, then Algorithm \ref{alg:new1} is $\epsilon$-DP or $(\epsilon,\delta)$-DP based on different stepsizes, noises and $\{\gamma_k\}_{k=1}^m$, under the assumption that $n$ is sufficiently large such that $\gamma_0 \geq \frac{\|\mathcal{W}\|_2}{L}$, where $\|\mathcal{W}\|_2$ is the diameter of the set $\mathcal{W}$, {\em i.e.,} $\|\mathcal{W}\|_2=\max_{w, w'\in \mathcal{W}}\|w-w'\|_2$. Moreover, for $(\epsilon, \delta)$-DP, we have 
\begin{equation*}
    \min_{k=1, \cdots, m} \mathbb{E}[F(w_k)]-\min_{w\in \mathcal{W}}F(w)=O\left( \left (\frac{ L^2}{\lambda^\frac{2}{\theta}} \left(\frac{\log n}{n}+\frac{d\log n^2\log(1/\delta)}{n^2\epsilon^2}\right)\right)^\frac{\theta}{2(\theta-1)}\right).  
\end{equation*}
For $\epsilon$-DP we have 
    $\min_{k=1, \cdots, m} \mathbb{E}[F(w_k)]-\min_{w\in \mathcal{W}}F(w)=O\left(  \left(\frac{ L^2}{\lambda^\frac{2}{\theta}}\left (\frac{\log n}{n}+\frac{d^2\log n^2}{n^2\epsilon^2}\right)\right)^\frac{\theta}{2(\theta-1)}\right).  $
\end{theorem}

So far we have proposed two algorithms. However, there are still several issues: First, both of the previous methods need strong assumption on $\theta$. Algorithm \ref{alg:new1} needs $\theta$ to be known in advance while Algorithm \ref{alg:3} needs to assume $\theta\geq 2$. Thus, can we develop a method that only needs a weaker assumption on $\theta$? Secondly, both of the previous two algorithms could achieve rates of 	$\tilde{O}((\frac{1}{\sqrt{n}}+\frac{d}{n\epsilon})^\frac{\theta}{\theta-1}) $ and $\tilde{O}((\frac{1}{\sqrt{n}}+\frac{\sqrt{d\log(1/\delta)}}{n\epsilon})^\frac{\theta}{\theta-1})$ for $\epsilon$-DP and $(\epsilon, \delta)$-DP, respectively. Can we further improve these bounds? Thirdly, the two methods are either impractical or inefficient. Specifically, for Algorithm \ref{alg:new1}, as we can see from our theoretical analysis, we need to exactly set $\gamma_0$ as $\chi_0 / ({6400 L^2}(\frac{1}{n_0}+\frac{d\log(1/\delta)}{n_0^2\epsilon^2}))$ in the $(\epsilon, \delta)$-DP model, which is quite large and is difficult to get. And we can see that in Theorem \ref{new:th1} we can only guarantee there exists a $w_k$ that achieves the upper bound of error, it is still unknown how to find such $w_k$ privately with the same theoretical guarantees. For Algorithm \ref{alg:3}, as we will see in the experiment part, it even does not outperform the previous Phased-SGD method (Algorithm \ref{alg:1}), which means that its performance is quite poor. Thus, how to design improved methods both theoretically and practically? In the following we will focus on these three issues by developing a new method. 

The idea of our algorithm is as the following: assuming that the value of $\theta$ is unknown, but $\theta$ is lower bounded by some known constant $\bar{\theta}>1$, namely $\theta\geq \bar{\theta}>1$. We first divide the whole dataset into $k=\lfloor (\log_{\bar{\theta}}2)\cdot \log\log n\rfloor$ disjoint subsets, where the $i$-th subset has $n_i=2^{i-1} n/(\log n)^{\log_{\bar{\theta}}2}$ samples;  then we repeat the Algorithm \ref{alg:1} for $k$ times where each phase runs on the $i$-th subset and is initialized at the output of the previous phase. See Algorithm \ref{alg:2} for details. 	\begin{algorithm}
	\caption{Iterated Phased-SGD($w_1, n,\mathcal{W}, \bar{\theta}$) \label{alg:2}}
		\KwIn{  	Dataset $S=\{x_1,\cdots,x_n\} $, convex function $f:\mathcal{W}\times \mathcal{X}\mapsto \mathbb{R}$, initial point $w_0\in\mathcal{W}$, privacy parameter $\epsilon$ and (or) $\delta$. $D$ is a constant satisfying $D\geq \|w_0-w^*\|_2$.}
	    Partite the data $S$ into $k$ disjoint subsets $\{S_1, \cdots, S_k\}$, where $k=\lfloor (\log_{\bar{\theta}}2)\cdot \log\log n\rfloor$ and for each $i\in [k]$, $|S_i|=n_i=\lfloor 2^{i-1} n/(\log n)^{\log_{\bar{\theta}}2}\rfloor$.\;
	    
		\For {$t=1, \cdots, k$}{
		Let $w_{t}=\text{Phased-SGD}(w_{t-1}, \eta_t, n_t, \mathcal{W})$, where the Phased-SGD runs on the $t$-th subset $S_i$. For $(\epsilon,\delta)$-DP, $\eta_t=\frac{D}{L}\min\{\frac{4}{\sqrt{n_i}}, \frac{\epsilon}{2\sqrt{d\log(1/\delta)}}\}$. For $\epsilon$-DP,  $\eta_t=\frac{D}{L}\min\{\frac{4}{\sqrt{n_i}}, \frac{\epsilon}{d}\}$. \;
		}
		\Return {$w_{k}$}
\end{algorithm}
\begin{remark}
Although both Algorithm \ref{alg:2} and \ref{alg:3} partite the data into several parts and perform the Phased-SGD several times. There are several differences: First, the sizes of subsets in Algorithm \ref{alg:3} are equal, while we partite the data aggressively in Algorithm \ref{alg:2}. Secondly, in each phase of Algorithm \ref{alg:2}, the convex set to be projected is invariant while in Algorithm \ref{alg:3} we constantly replace it to $\mathcal{W}\cap \mathbb{B}(\hat{w}_{k-1},R_{k-1})$, which is necessary based on our theoretical analysis.  
\end{remark}
	\begin{theorem}\label{thm:1}
	Assume that $F(\cdot)$ is $(\theta, \lambda)$-TNC with $\theta\geq \bar{\theta}>1$ for some known constant $\bar{\theta}$, and $f(\cdot, x)$ is  convex, $\beta$-smooth and $L$-Lipschitz for each $x$. If the sample size $n$ is sufficiently large such that $\bar{\theta}\geq 2^{\frac{\log\log n}{(\log n)-1}}$, then Algorithm \ref{alg:2} is either $\epsilon$-DP or $(\epsilon, \delta)$-DP for any $\epsilon\leq 2\log(1/\delta)$, based on different step sizes and noises under the assumption that $\eta_t\leq \frac{1}{\beta}$. Moreover, for $(\epsilon, \delta)$-DP, the output satisfies
$\mathbb{E} [F(w_k)]-\min_{w\in \mathcal{W}} F(w)= O\left(\left(\frac{L^{\theta}}{\lambda}\right)^{\frac{1}{\theta -1}}\cdot\left(\frac{1}{\sqrt{n}}+\frac{\sqrt{d\log(1/\delta)}}{\epsilon n}\right)^{\frac{\theta}{\theta-1}}\right). $
		For $\epsilon$-DP, we have $\mathbb{E} [F(w_k)]-\min_{w\in \mathcal{W}} F(w)= O\left(\left(\frac{L^{\theta}}{\lambda}\right)^{\frac{1}{\theta -1}}\cdot\left(\frac{1}{\sqrt{n}}+\frac{d}{\epsilon n}\right)^{\frac{\theta}{\theta-1}}\right). $
	\end{theorem}

\begin{remark}\label{remark1}
Compared with the previous results, we can see the upper bounds in Theorem \ref{thm:1} improve factors of $\text{Poly}(\log n)$ in both $\epsilon$-DP and $(\epsilon, \delta)$-DP models. Moreover, instead of $\theta\geq 2$, we only need the assumption of $\theta\geq \bar{\theta}$ for some known $\bar{\theta}>1$ in  Theorem \ref{thm:1}. And as we will see in the experimental part, Algorithm \ref{alg:2} outperforms all previous methods. 
\end{remark}
\begin{remark}\label{remark:3}
    We can see that it is possible to get faster rates than the rates of strongly convex loss if $\theta<2$. For example, when $\theta=\frac{3}{2}$, the upper bound of error becomes  $O((\frac{1}{\sqrt{n}}+\frac{\sqrt{d\log(1/\delta)}}{\epsilon n})^3)$ in the $(\epsilon, \delta)$-DP model. Moreover, when $\theta>1$, then the bounds will be always higher than the optimal rate for general convex loss as $\frac{\theta}{\theta-1}>1$. When $\theta=2$, we have an excess population risk of $O((\frac{1}{\sqrt{n}}+\frac{\sqrt{d\log ({1}/{\delta})}}{n\epsilon})^{2})$ and $O((\frac{1}{\sqrt{n}}+\frac{d}{\epsilon n})^2)$ for $\epsilon$-DP and $(\epsilon, \delta)$-DP respectively, which matches the optimal rate of DP-SCO with strongly convex function \citep{feldman2020private}. Besides strongly convex  functions, there are other problems that satisfy $(2, \lambda)$-TNC, such as the functions satisfying Weak Strong Convexity, Restricted Secant Inequality (RSI), Error Bound (EB) and Polyak-{L}ojasiewicz (PL) conditions (see Section 2.1 in \citep{karimi2016linear} for details). Thus, Theorem \ref{thm:1} with $\theta=2$ could be seen as a generalization of the strongly convex case. For Polyak- {L}ojasiewicz (PL) functions, \citep{wang2018differentially} shows an upper bound of $O(\frac{d\log(1/\delta)}{n^2\epsilon^2})$ for the empirical risk. However, their method cannot be extended to the population risk. In the following  we provide some examples that satisfy TNC with $\theta=2$. 
\end{remark}
\iffalse
\begin{theorem}[Strongly convex composed with linear]\label{thm:2}
     If a one dimensional function $g(\cdot)$ is $\lambda$-strongly convex on an interval $[-B, B]$, then the population function 
\begin{equation}
    F(w)=\mathbb{E}[g(\langle x, w\rangle)] 
\end{equation}
will be $(2, \frac{\lambda \sigma(A)}{2})$-TNC if $\mathcal{W}$ and the distribution of $x$ satisfies that $|\langle w, x \rangle|\leq B$ for any $w\in \mathcal{W}$ and $x\in \mathcal{D}$, where $A=\mathbb{E}[xx^T]$ and $\sigma(A)$ is the minimal non-zero eigenvalue of the matrix $A$. 
\end{theorem}
\begin{proof}[{\bf Proof of Theorem \ref{thm:2}}]
    For any $w' \in \mathcal{W}$, denote $w=\arg\min_{u\in\mathcal{W_*}}\|u-w'\|_2^2$. By the assumption of the function $g$ we have 
    \begin{align*}
        g(\langle x, w\rangle)& \geq  g(\langle x, w'\rangle)+ g'(\langle x, w'\rangle)(\langle x, w-w' \rangle)+ \frac{\lambda}{2}|\langle x, w-w' \rangle|^2.
    \end{align*}
    Take the expectation w.r.t $x$ and by the definition of $\sigma(A)$ we have 
    \begin{align*}
        F(w)&\geq F(w')+ \langle \nabla F(w'), w-w'\rangle +\frac{\lambda \sigma(A)}{2}\|w-w'\|_2^2 \\
        &\geq \geq F(w')+ \min_{y\in\mathcal{W}}[\langle \nabla F(w'), y-w'\rangle +\frac{\lambda \sigma(A)}{2}\|y-w'\|_2^2] \\
        &= F(w')-
    \end{align*}
\end{proof}
\fi 
\begin{lemma}[Quadratic Problem \citep{liu2018fast}]\label{lemma:5}
    Consider the quadratic problem 
    $F(w)=w^T\mathbb{E}_x[A(x)]w+w^T\mathbb{E}_x[b(x)]+c,$
    where $c$ is a constant. If $\mathbb{E}[A(x)]$ is a positive {\bf semi-definite} matrix, the loss function $f(w, x)=w^TA(x)w+w^Tb(x)+c$ is Lipschitz (e.g., $\max\{\|A(x)\|_2, \|b(x)\|_2\}\leq O(1)$) and $\mathcal{W}$ is a bounded polyhedron (e.g., $\ell_1$-norm or $\ell_\infty$-norm ball), then the population risk function will be TNC with $\theta=2$ and the problem will satisfy the assumptions in Theorem \ref{thm:1}. %However, it is notable that the rate will not be optimal in the case where $\mathcal{W}$ is the $\ell_1$-norm ball, as it has been shown that it is possible to achieve a rate that is only dependent on logarithm of the dimension $d$ \citep{asi2021private}.  
\end{lemma}
By Lemma \ref{lemma:5} we can see that for the linear regression problem where $F(w)=\mathbb{E}(x^Tw-y)^2$ over a bounded polyhedron $\mathcal{W}$. It is possible to achieve an upper bound of $O(\frac{1}{n}+\frac{d\log (1/\delta)}{n^2\epsilon^2})$ and $O(\frac{1}{n}+\frac{d^2}{n^2\epsilon^2})$ for the excess population risk in the $(\epsilon, \delta)$-DP and $\epsilon$-DP model, respectively. 
\begin{lemma}[SCO over $\ell_2$-norm ball \citep{liu2018fast}]
    Consider the problem of SCO over $\ell_2$-norm ball 
    $ \min_{\|w\|_2\leq B} F(w)=\mathbb{E}[f(w, x)]. 
    $
    If $f(\cdot, x)$ is convex, smooth and Lipschitz, and $\min_{w\in \mathbb{R}^d} F(w)< \min_{\|w\|_2\leq B} F(w)$. Then the population risk is TNC with $\theta=2$ and satisfies the assumptions in Theorem \ref{thm:1}. 
\end{lemma}
So far, we provided several methods for TNC population risk functions under the assumption that the loss function is smooth. We can extend the previous methods to the non-smooth case, see Appendix \ref{sec:nonsmooth} for details. 
% In the previous part we showed that if the population risk function is TNC, then it is possible to achieve faster rates than the general convex ones, and even for strongly convex ones. However, a critical issue in Algorithm \ref{alg:1} is that we need to know the TNC parameter $\theta$ in advance, which is difficult to get in practice. Thus, our question is, can we provide a more practical method where the TNC parameter $\theta$ is unknown? In the following we provide an affirmative answer for the case where $\theta\geq 2$.
%\vspace{-0.1in}
\subsection{Lower Bounds of Excess Population Risk}
%\vspace{-0.1in}
	In the previous section, we provide an algorithm whose output could achieve an excess population risk of $O((\frac{1}{\sqrt{n}}+\frac{\sqrt{d\log(1/\delta)}}{n\epsilon})^{\frac{\theta}{\theta-1}})$ and $O((\frac{1}{\sqrt{n}}+\frac{d}{\epsilon n})^{\frac{\theta}{\theta-1}})$ for $\epsilon$-DP and $(\epsilon, \delta)$-DP respectively. The question is, can we further improve these bounds? 
	In this section, we show that  for all $\theta\geq 2$, the term of $O((\frac{\sqrt{d\log(1/\delta)}}{n\epsilon})^{\frac{\theta}{\theta-1}})$ and $O((\frac{d}{\epsilon n})^{\frac{\theta}{\theta-1}})$ cannot be further improved. We consider the following loss function. Define
	\begin{equation}\label{eq:7}
	    f(w, x)=-\langle w, x \rangle+ \frac{1}{\theta}\|w\|_2^\theta, \|w\|_2\leq 1, x\in \{-\frac{1}{\sqrt{d}}, \frac{1}{\sqrt{d}}\}^d. 
	\end{equation}

	\begin{theorem}[Lower bound of $(\epsilon,\delta)$-DP ]\label{thm:3}
		Let $n,d\in \mathbb{N}$, $\theta\geq 2$, $\epsilon>0$  and $\delta=o(\frac{1}{n})$ such that $n\geq \Omega (\frac{\sqrt{d\log(1/\delta)}}{\epsilon})$. For every $(\epsilon, \delta)$-Differentially Private algorithm, there is a dataset $S=\{x_1, \cdots, x_n\}$ where each $x_i\in \{-\frac{1}{\sqrt{d}}, \frac{1}{\sqrt{d}}\}^d$ such that with probability at least $\frac{1}{3}$ over the randomness of the algorithm, its output $w_{priv}$ satisfies
			\begin{equation*}
		F(w_{priv})-\min_{\|w\|_2\leq 1}F(w) = \Omega \left(     
	(\frac{\sqrt{d\log(1/\delta)}}{n\epsilon})^{\frac{\theta}{\theta-1}}
		\right), 
		\end{equation*}
		where loss function  is given by (\ref{eq:7})
		which is $O(1)$-Lipschitz, and the population risk  satisfies $(\theta, O(1))$-TNC. 
	\end{theorem}

		\begin{theorem}[Lower bound of $\epsilon$-DP ]\label{thm:4}
		Let $n,d\in \mathbb{N}$, $\theta\geq 2$ and $\epsilon>0$ such that $n\geq \Omega (\frac{\sqrt{d}}{\epsilon})$. For every $\epsilon$-Differentially Private algorithm, there is a dataset $S=\{x_1, \cdots, x_n\}$ where each $x_i\in \{-\frac{1}{\sqrt{d}}, \frac{1}{\sqrt{d}}\}^d$ such that with probability at least $\frac{1}{3}$ over the randomness of the algorithm, its output $w_{priv}$ satisfies
			\begin{equation*}
		F(w_{priv})-\min_{\|w\|_2\leq 1}F(w) =\Omega\left(     \left(\frac{d}{n\epsilon}\right)^{\frac{\theta}{\theta-1}}
		\right), 
		\end{equation*}
		where loss function  is given by (\ref{eq:7})
		which is $O(1)$-Lipschitz, and the population risk  satisfies $(\theta, O(1))$-TNC. 
	\end{theorem}
		\begin{algorithm}
	\caption{Epoch-DP-SGD($\eta_1,n_1,n,w_0$) \label{alg:5}}
	\KwIn{ 	Parameter  $\lambda$, dataset $S=\{x_1,\cdots,x_n\} $, convex function $f:\mathcal{W}\times \mathcal{X}\mapsto \mathbb{R}$, the first partition $n_1$, initial point $w_0\in\mathcal{W}$, privacy parameter $\epsilon$ and (or) $\delta$. }
Set $k=\lceil \log \frac{n}{2n_1}+1\rceil$ and partite the whole dataset into $\{S_1,S_2,\cdots,S_k\}$. Denote $n_i=|S_i|$, which satisfies $n_{i+1}=2n_{i}$ (if there are left samples, we will add them to the last subset).\;

\For {$i=1,\cdots,k$
	}{
	Set $w_i^1=w_{i-1}$. \;
		\For{$t=1,\cdots,n_i$}{
		Update	\begin{equation}\label{eq6}
	w_i^{t+1}=\prod_{\mathcal{W}}(w_i^{t}-\eta_i\nabla_{w}f(w_{i}^{t},x_i^t)), 
			\end{equation}
			where $x_i^t$ is the $t$-th sample in the set $S_i$. \;
		}
	  	Update  $\overline{w}_i=\frac{1}{n_i+1}\sum \limits_{t=1}^{n_i+1} w_i^t$.\;
		
		Let $w_i=\overline{w}_i+\xi_i$, where $\xi_i \sim \mathcal{N}(0,\sigma_i^2\mathbb{I}_d)$ with $\sigma_i=\frac{4L^2\sqrt{\log(1/\delta)}}{n_i \epsilon \lambda}$ for $(\epsilon,\delta)$-DP and $\xi_i=(\zeta_1, \cdots, \zeta_d)$ with each $\zeta_j \sim \text{Lap}(\lambda)$ and $\lambda=\frac{4L^2\sqrt{d}}{ \lambda n_i\epsilon}$  for $\epsilon$-DP. \;
		Set $\eta_{i+1}=\eta_i/2$. \;
	}
		\Return{$w_k$}
\end{algorithm}		
\begin{remark}
From the above theorems, we can see that  for the case where $\theta=2$, the loss function in (\ref{eq:7}) is reduced to the squared loss, which was used to the lower bound proof of strongly convex loss in \citep{bassily2014private}. 
\end{remark}	
	
\section{Improved Rates for Strongly Convex Loss}
In the previous section, we showed upper and lower bounds  of the excess population risk for general TNC population risk functions. Moreover, from Theorem \ref{thm:5} we can see that we get asymptotically the same bound for smooth and non-smooth loss functions in the $(\epsilon, \delta)$-DP model. However, in the non-private case, it has been shown that for the strongly convex loss functions, it is possible to get an improved rate compared with the non-smooth ones \citep{zhang2019stochastic}. Thus, our question is, can we get improved rates if the loss functions have additional properties? In the following we will  study the strongly convex loss case. Specifically, we will show that when the loss function $f(\cdot, x)$ has additional assumptions on non-negativity and if the optimal value $F(w^*)$ is sufficiently small, it is possible to achieve an upper bound of $O(\frac{d\log(1/\delta)}{n^2\epsilon^2}+\frac{1}{n^{\tau}})$ for any $\tau>1$ if the sample size $n$ is sufficiently large.

There are two parts in the algorithm. In the first part, we perform the original Iterated Phased-SGD (Algorithm \ref{alg:2}) on the first half of the data to get a good solution to the optimal parameter $w^*$. After that we perform a new method, namely Epoch-DP-SGD (Algorithm \ref{alg:5}) on the second half of the data, which may also be used in other problems. We note that although Algorithm \ref{alg:5} and Algorithm \ref{alg:1} both perform the original DP-SGD algorithm  in \citep{bassily2014private} for several phases or epochs. They are quite different: First, as the phase/epoch increases, we decrease the size of the subset (or the number of iterations) in Algorithm \ref{alg:1}. While in Algorithm \ref{alg:5} we will increase the size of the subset (or the number of iterations). As we will see in the proof, this increase is necessary.  Secondly, the initial size of the subset in Algorithm \ref{alg:1} is $\frac{n}{2}$ while it is $2^{2\tau+3}\kappa$ in Algorithm \ref{alg:5}, where $\kappa$ is the condition number $\kappa=\beta/\lambda$ of the population risk functions.

		\begin{algorithm}
	\caption{Faster-DPSGD-SC\label{alg:6}}
		\KwIn{	Parameter $\beta$, $\lambda$, $\kappa=\frac{\beta}{\lambda}$ and $\tau$. Dataset $S=\{x_1,\cdots,x_n\} $, convex function $f:\mathcal{W}\times \mathcal{X}\mapsto \mathbb{R}$, initial point $w_0\in\mathcal{W}$, privacy parameter $\epsilon$ and (or) $\delta$. }
		
	    Split the dataset $S$ into $S_1, S_2$ where $|S_1|=|S_2|=\frac{n}{2}$. \;
	    
Perform Iterated Phased-SGD($w_0,\frac{n}{2},\mathcal{W}$) with $\theta=2$ on $S_1$. Denote the returned solution as $\hat{w}$.\;

Perform Epoch-DP-SGD($\frac{1}{4\beta},2^{2\tau+3}\cdot \kappa,\frac{n}{2},\hat{w}$) on $S_2$. Denote the returned solution by $\tilde{w}$.\;

	\Return {$\tilde{w}$}
\end{algorithm}

	\begin{theorem}\label{thm:6}
 Given $\epsilon$ and $\delta$, if $f(\cdot, x)$ is convex, $L$-Lipschitz and $\beta$-smooth for all $x$, Algorithm \ref{alg:6} is either $\epsilon$-DP or $(\epsilon,\delta)$-DP, based on different choices on the stepsizes and noises, under the assumption that $\eta_k\leq \frac{2}{\beta}$ in Algorithm \ref{alg:2}. 
\end{theorem}
\begin{theorem}\label{thm:7}
Denote $\min_{w\in\mathcal{W}}F(w)=F(w^*)$ and suppose $n\geq \kappa ^{\tau}$ for some constant  $\tau>1$, and $F(w)$ is $L$-Lipschitz, $\lambda$-strongly convex and $\beta$-smooth. For $(\epsilon,\delta)$-DP, 
	the output returned by algorithm \ref{alg:6} satisfies
	\begin{equation*}
	\mathbb{E}[F(\tilde{w})]-F(w^*)=O\left( \frac{L^4 \beta d\log(1/\delta)}{\lambda^2 n^2\epsilon^2}+\frac{4^{\tau}\cdot \kappa  F(w^{*})}{n} +\frac{ L^2}{\lambda}\left(\frac{2^{2\tau^2+4\tau}}{n^{\tau}}+\frac{2^{4\tau^2+10\tau  }\cdot d \log(1/\delta)}{n^{2\tau}\cdot \epsilon^2}\right)  \right).
	\end{equation*}
	Specifically, when $\tau=\log_{\kappa} n$, we have for any $n$,
	\begin{multline*}
	\mathbb{E}[F(\tilde{w})]-F(w^*)= O\big( \frac{L^4 \beta d\log(1/\delta)}{\lambda^2 n^2\epsilon^2}+\frac{\kappa  F(w^{*})}{n^{1-2\log_\kappa 2} } + \\\frac{ L^2}{\lambda}\left(\frac{1}{n^{(1-4\log_\kappa 2)\log_\kappa n- 4\log_\kappa 2}}+\frac{2^{4\tau^2+10\tau  }\cdot d \log(1/\delta)}{n^{(2-4\log_\kappa 2- 10\log_\kappa 2)\log_\kappa n }\cdot \epsilon^2}\right)  \big).
	\end{multline*}
	
	For $\epsilon$-DP, 
	the output returned by algorithm \ref{alg:6} satisfies
	\begin{equation*}
	\mathbb{E}[F(\tilde{w})]-F(w^*)=O\left( \frac{L^4 \beta d^2}{\lambda^2 n^2\epsilon^2}+\frac{4^{\tau}\cdot \kappa  F(w^{*})}{n} +\frac{ L^2}{\lambda}\left(\frac{2^{2\tau^2+4\tau}}{n^{\tau}}+\frac{2^{4\tau^2+10\tau  }\cdot d^2}{n^{2\tau}\cdot \epsilon^2}\right)  \right).
	\end{equation*}
\end{theorem}

We note that recently \citep{wang2020differentially} also showed that when  the loss function is non-negative and the optimal value of the population risk is small, it is possible to get a non-trivial upper bound for DP-SCO. However, there are some differences: Firstly, \citep{wang2020differentially} only studied the case of DP-SCO with heavy-tailed data while here we study DP-SCO with strongly convex functions. Thus, the problems are different. Moreover, their method is based on the sample-and-aggregate framework, which is impractical, and their result is $O(\frac{d^3 F(w^*)}{n\epsilon^4})$ under the assumption that $\nabla F(w^*)=0$, which may not hold in the case where $\mathcal{W}$ is a close set. Compared with their work, we do not need such strong assumption and in general our bound is much smaller than theirs for $F(w^*)=O(1)$. 
\begin{remark}
Theorem \ref{thm:7} implies that when $n=\Omega(\kappa^\tau)$, the output of Algorithm \ref{alg:6} achieves excess population risks of $O(\frac{d\log(1/\delta)}{n^2\epsilon^2}+\frac{ F(w^*)}{n}+\frac{1}{n^\tau})$ and $O(\frac{d^2}{n^2\
\epsilon^2}+\frac{ F(w^*)}{n}+\frac{1}{n^\tau})$ for $(\epsilon, \delta)$-DP and $\epsilon$-DP, respectively, which are faster than the optimal rates of $O(\frac{1}{n}+\frac{d\log(1/\delta)}{n^2\epsilon^2})$ and $O(\frac{d^2}{n^2\
\epsilon^2}+\frac{1}{n})$ for general strongly convex loss functions, under the assumption that the optimal risk $F(w^*)$ is relatively small. It is also notable that the bounds in Theorem \ref{thm:7} have exponential dependence on the parameter $\tau$, which means $\tau$ also cannot be very large. Moreover, due to the large (hidden) constant in the upper bound, the practical performance of Theorem \ref{thm:7} is poor. We leave the problem of designing more practical algorithms for future research. 
\end{remark}
 
 \section{Conclusion}
 In this paper, we studied DP-SCO with special classes of population functions. In the first part of the paper, we study the case where the population function satisfies TNC with the parameter $\theta>1$. Specifically, we first provided several methods which could achieve upper bounds of $\tilde{O}((\frac{1}{\sqrt{n}}+\frac{d}{n\epsilon})^\frac{\theta}{\theta-1}) $ and $\tilde{O}((\frac{1}{\sqrt{n}}+\frac{\sqrt{d\log(1/\delta)}}{n\epsilon})^\frac{\theta}{\theta-1})$ for $\epsilon$-DP and $(\epsilon, \delta)$-DP, respectively. Then we showed that for any $\theta>1$, there is a population risk function satisfies TNC with $\theta$ such that for any $\epsilon$-DP ($(\epsilon, \delta)$-DP) algorithm, the excess population risk of its output is lower bounded by $\Omega((\frac{d}{n\epsilon})^\frac{\theta}{\theta-1}) $ and $\Omega((\frac{\sqrt{d\log(1/\delta)}}{n\epsilon})^\frac{\theta}{\theta-1})$ for $\epsilon$-DP and $(\epsilon, \delta)$-DP, respectively. In the second part of the paper, we revisited DP-SCO with strongly convex loss functions. We claimed that when the loss function is non-negative and the optimal value of the population function is small enough, it is possible to achieve an upper bound of $O(\frac{d\log(1/\delta)}{n^2\epsilon^2}+\frac{1}{n^{\tau}})$ and $O(\frac{d^2}{n^2\epsilon^2}+\frac{1}{n^{\tau}})$ for any $\tau> 1$ in $(\epsilon,\delta)$-DP and $\epsilon$-DP model respectively if the sample size $n$ is sufficiently large. 
 
 Besides the open problems we mentioned in the previous parts, there are other unsolved problems: 1) From the theoretical results in this paper, we can see there is still a gap of $O(\frac{1}{n^\frac{\theta}{2(\theta-1)}})$ between upper bounds and lower bounds in both $\epsilon$-DP and $(\epsilon, \delta)$-DP models. Thus, the optimal rates of excess population risk is still unknown. 2) In this paper we provide faster rates of DP-SCO with special class of functions, especially for TNC population functions. However, besides TNC, there are other special classes of functions which have faster rates in the non-private case, such as exponential concave loss \citep{koren2015fast}. It is still unknown whether we can get faster rates under the differential privacy constraint. We will leave these problems for future research. 

\section*{Acknowledgements}
 Di Wang, Lijie Hu and Jinyan Su were support in part by the baseline funding BAS/1/1689-01-01 and funding from the AI Initiative REI/1/4811-10-01 of King Abdullah University of Science and Technology (KAUST).

    \bibliography{acmart}
   \newpage 
\appendix 

\section{Extension to Non-smooth Loss}\label{sec:nonsmooth}
In the previous section, we provided several methods for TNC population risk functions under the assumption that the loss function is smooth. However, we constantly meet the case where the loss is non-smooth. In this section, we will extend the previous methods to the non-smooth case. The observation is that, in both Algorithm \ref{alg:2} and Algorithm \ref{alg:3}, we use the Phased-SGD (Algorithm \ref{alg:1}) as a sub-routine for several phases. And we need the smoothness condition in Phased-SGD to get the upper bounds in Lemma \ref{lemma:4}. Thus, to extend to the non-smooth case, the most direct way is to change  Phased-SGD to a non-smooth version in both Algorithm \ref{alg:2} and Algorithm \ref{alg:3}.  \citep{feldman2020private} provided non-smooth version of Phased-SGD based on proximal mapping for $(\epsilon, \delta)$-DP model, namely Phased-ERM, which is shown in Algorithm \ref{alg-Phased-EMR}. 
	\begin{algorithm}
	\caption{Phased-ERM($w_0, \eta,n,\mathcal{W}$) algorithm \citep{feldman2020private} \label{alg-Phased-EMR}}
		\KwIn{  	Dataset $S=\{x_1,\cdots,x_n\} $, convex function $f:\mathcal{W}\times \mathcal{X}\mapsto \mathbb{R}$, initial point $w_0\in\mathcal{W}$, step size $\eta$ (will be specified later), privacy parameters $\epsilon, \delta$.}
			Set $k= \lceil \log_2 n\rceil$. Partite the whole dataset into $k$ subsets $\{S_1, \cdots, S_k\}$ where $|S_i|=\lfloor 2^{-i}n \rfloor$. \;
			
		\For {$i=1,\cdots ,k$}{
		 Let  $n_i=2^{-i}n$, $\eta_i=4^{-i}\eta$. \;
		 
		Compute $\tilde{w}_i\in\mathcal{W}$ such that $F_i(\tilde{w}_i)-\min_{w\in\mathcal{W}}F_i(w)\leq \frac{L^2\eta_i}{n_i}$ with probability at least $1-\delta$ for 
		\begin{equation*}
		    F_i(w)=\frac{1}{n_i}\sum_{x\in S_i}f(w, x)+\frac{1}{\eta_i n_i}\|w-w_{i-1}\|_2^2. 
		\end{equation*} \;
        Set $w_i=\tilde{w}_i+\xi_i,$ where $\xi_i\sim \mathcal{N}(0, \sigma_i\mathbb{I}_d)$ with $\sigma_i=\frac{4L\eta_i\sqrt{\log(1/\delta)}}{\epsilon}$.  \;
		}
		\Return {$w_k$}
\end{algorithm}
\begin{lemma}[Theorem 4.8 in \citep{feldman2020private}]
    Set $\eta=\frac{D}{L}\min\{\frac{4}{\sqrt{n}}, \frac{\epsilon}{2\sqrt{d\log(1/\delta)}}\}$. Then for the output of Algorithm \ref{alg-Phased-EMR} we have 
    \begin{equation*}
	\mathbb{E}[F(\hat{w})]-\min_{w\in \mathcal{W}}F(w)\leq 10LD\left(\frac{1}{\sqrt{n}}+\frac{\sqrt{d\log(1/\delta)}}{ n\epsilon}\right). 
    \end{equation*}
\end{lemma}
By using Algorithm \ref{alg-Phased-EMR} as subroutine in Algorithm \ref{alg:2} and \ref{alg:3} we have the following result, which is similar to Theorem \ref{thm:1} and \ref{thm:2}. 
\begin{theorem}\label{thm:5}
   	Assume that $F(\cdot)$ is $(\theta, \lambda)$-TNC and $f(\cdot, x)$ is  convex and $L$-Lipschitz for each $x$. For any $0<\epsilon, \delta<1$, if we replace the Phased-SGD with Phased-ERM in Algorithm \ref{alg:2} and \ref{alg:3} (we also change the stepsizes), then the two algorithms are $(\epsilon, \delta)$-DP. Moreover, in Algorithm \ref{alg:2},  the output satisfies
		\begin{equation*}
	\mathbb{E} [F(w_k)]-\min_{w\in \mathcal{W}} F(w)= O\left(\left(\frac{L^{\theta}}{\lambda}\right)^{\frac{1}{\theta -1}}\cdot\left(\frac{1}{\sqrt{n}}+\frac{\sqrt{d\log(1/\delta)}}{ n\epsilon }\right)^{\frac{\theta}{\theta-1}}\right). 
		\end{equation*} 
		If $n$ is sufficiently large such that $n\geq256$,  in Algorithm \ref{alg:3},  the output satisfies
		\begin{equation*}
\mathbb{E}[F(\hat{w}_m)]-F(w^*)= O\left(\left(\frac{L^\theta}{\lambda}\right)^{\frac{1}{\theta-1}}\cdot\left(\frac{\sqrt{\log n}}{\sqrt{n}}+\frac{\sqrt{d\log(1/\delta)}\log n}{n\epsilon}\right)^{\frac{\theta}{\theta-1}}\right).
\end{equation*}
\end{theorem}
In the following we provide some examples that satisfy TNC  with $\theta=2$ and with non-smooth loss. 
\begin{lemma}[Hinge Loss \citep{xu2017stochastic}]
    Consider the SVM problem with hinge loss
    \begin{equation*}
        \min_{w\in\mathcal{W}}F(w)=\mathbb{E}[(1-y\langle w ,x \rangle)_{+}], 
    \end{equation*}
\end{lemma}
where $\mathcal{W}$ is an $\ell_1$-norm or $\ell_\infty$-norm ball and $|\langle w, x\rangle|\leq 1$ for all $x$ and $w\in \mathcal{W}$. Then $F(\cdot)$ satisfies TNC with $\theta=2$. 
\begin{lemma}[$\ell_1$-regularized Problems  \citep{xu2017stochastic}]
    Consider the following $\ell_1$-regularized problem
    \begin{equation*}
        \min_{\|w\|_1\leq B}F(w)=\mathbb{E}[f(w, x)]+\lambda\|w\|_1,
    \end{equation*}
    where $\mathbb{E}[f(w, x)]$ is convex quadratic or piecewise linear, then $F(w)$ satisfies TNC with $\theta=2$. 
\end{lemma}
 \begin{figure*}[!htbp]
\centering
\subfigure[a8a]{
\includegraphics[width=5.5cm]{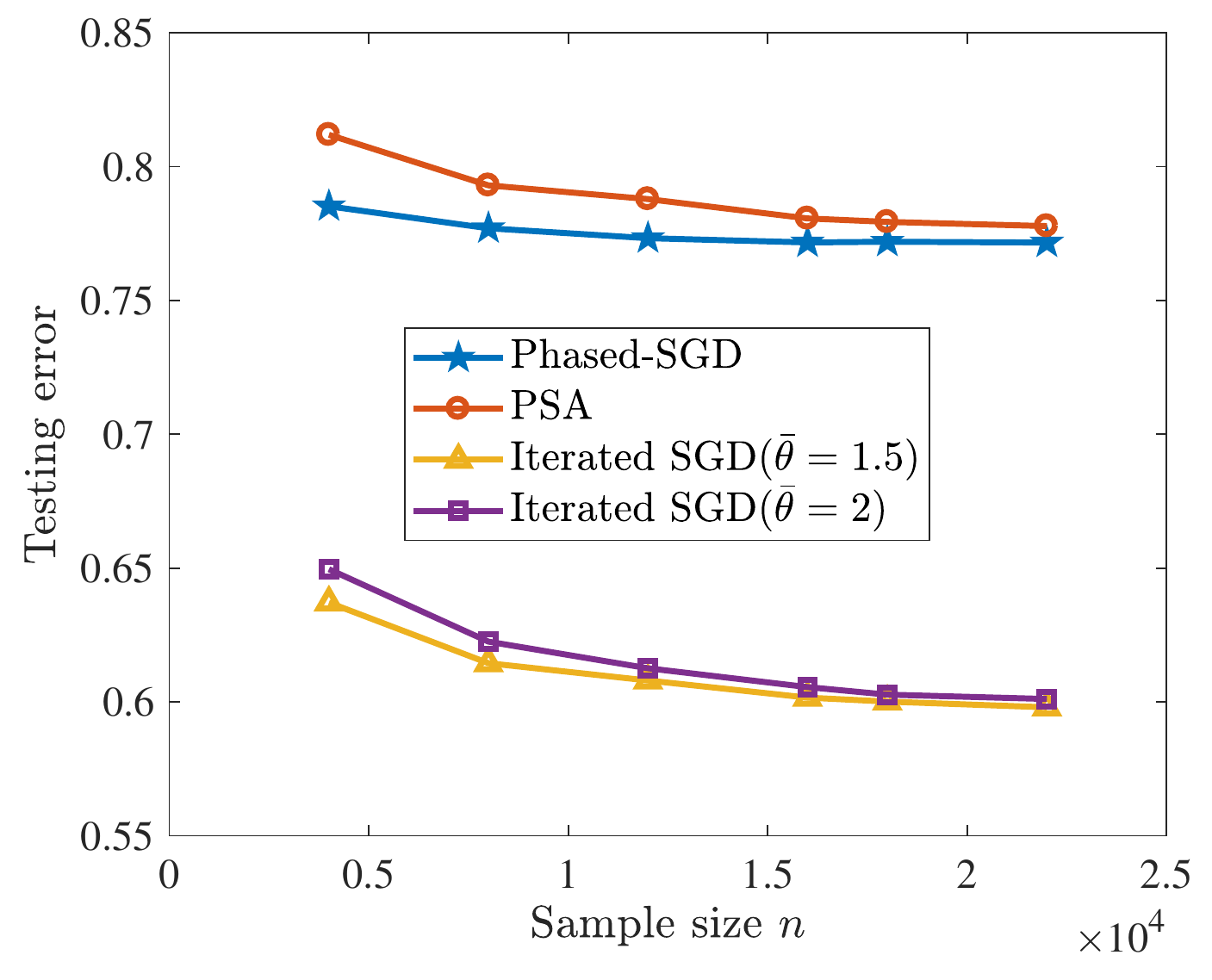}
%\caption{fig1}
}
\quad
\subfigure[a9a]{
\includegraphics[width=5.5cm]{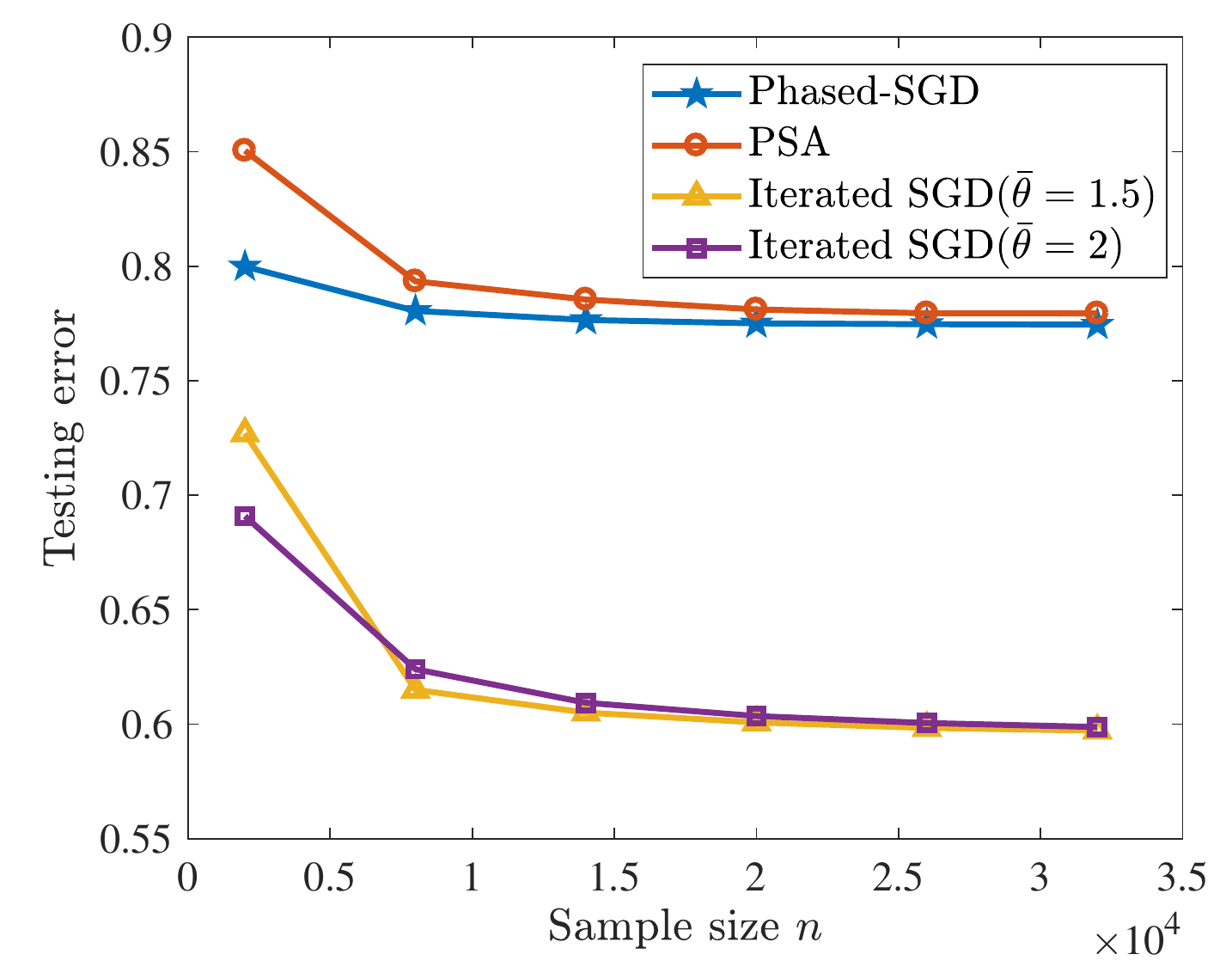}
}
\quad
\subfigure[ijcnn1]{
\includegraphics[width=5.5cm]{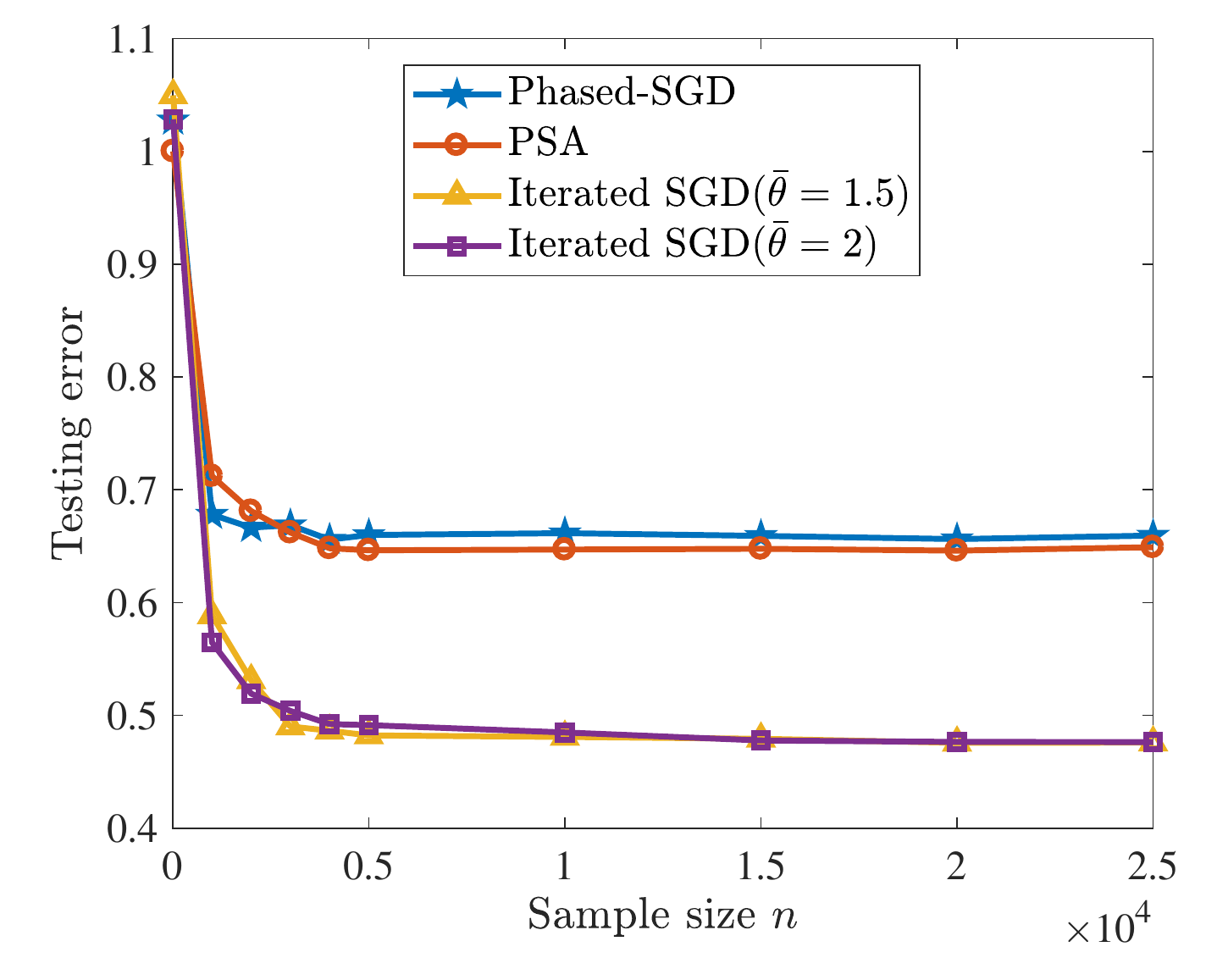}
}
\quad
\subfigure[w7a]{
\includegraphics[width=5.5cm]{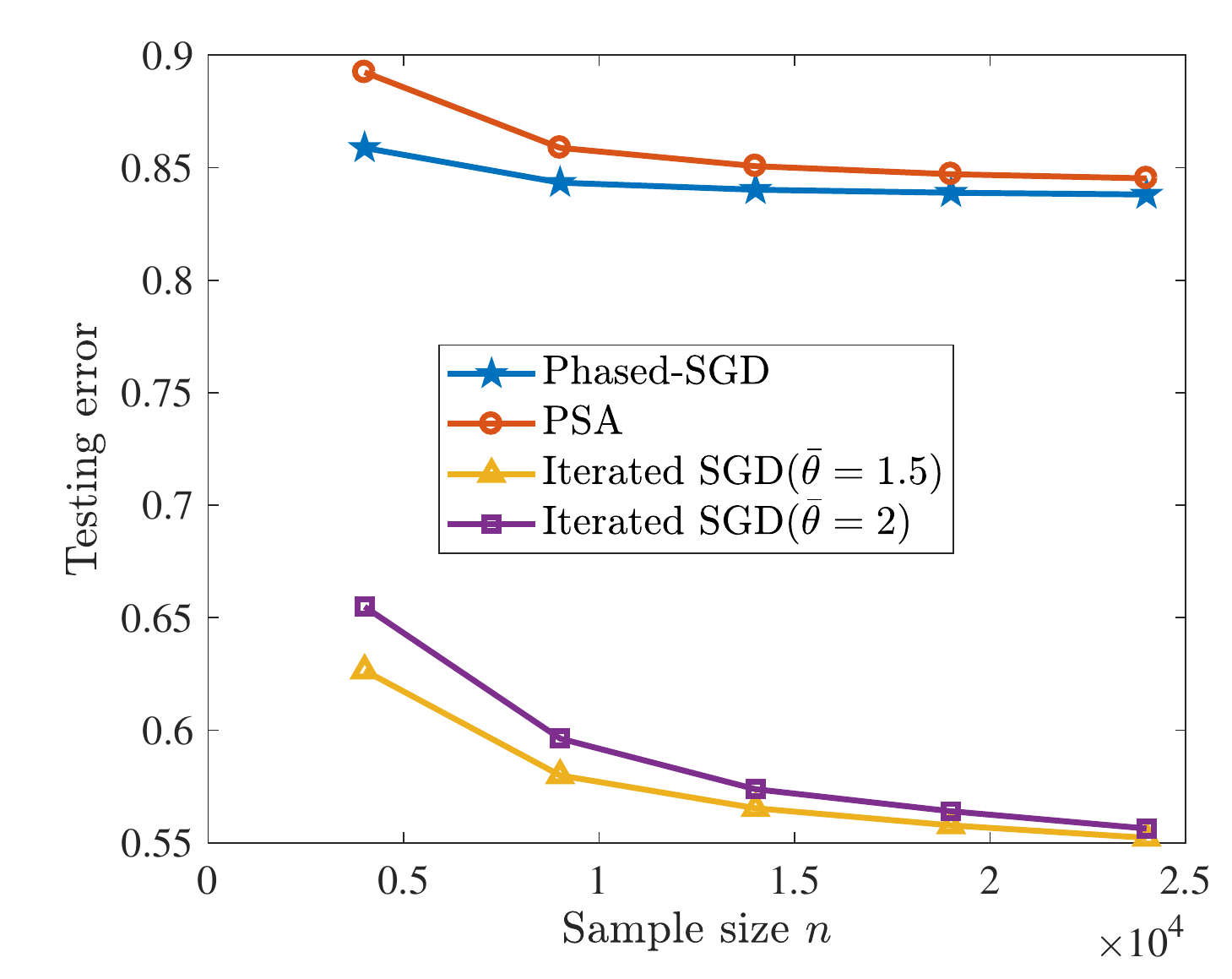}
}
\caption{ Results of Linear regression problem (\ref{eq23}) with  different training sample size. }
\label{fig1}
\end{figure*}
\section{Experiments}\label{sec:experiments}
 In this section, we  provide some experimental studies to compare the effectiveness of the proposed methods for population risk functions satisfying TNC. 
  \subsection*{Experimental Settings and Baseline Methods}
 For the problems satisfying TNC, we first consider the population risk with squared loss, which is  mentioned in Lemma \ref{lemma:5}, 
  \begin{equation}\label{eq23}
 \min \limits_{||w||_1\leq B}F(w)\overset{\Delta}{=} \mathbb{E}[(\langle w, x\rangle -y)^2]. 
 \end{equation}
 As we mentioned in the previous section, this problem satisfies TNC with parameter $\theta=2$. For this problem, we use Phased-SGD (Algorithm \ref{alg:1}) as our baseline method, which could be seen as the state-of-the-art method \citep{feldman2020private}. We will use PSA (Algorithm \ref{alg:3}) and iterated SGD (Algorithm \ref{alg:2}) with $\bar{\theta}=1.5$ and $2$ for comparisons. Note that here we will not compare with Algorithm \ref{alg:new1}. As we can see that, theoretically it involves quite large constants which impedes the algorithm to be practical. 
 
 Moreover, since strongly convex functions satisfy $(2,\lambda)$-TNC, we also perform our methods on strongly convex function. Here we will use Phased-SGD-SC (Algorithm \ref{alg:new}) \citep{feldman2020private}, which is known to have the optimal rate of error, as the baseline method. For this case, we will consider  the population risk function with squared logistic loss  and an additional  $l_2$-norm regularization: 
  \begin{equation}\label{eq24}
 \min \limits_{||w||_2\leq B}F(w)\overset{\Delta}{=} \mathbb{E}[\log (1+e^{-y\langle x, w\rangle})]+\frac{\lambda}{2}||w||^2_2, 
 \end{equation}
 which is $\lambda$-strongly convex and satisfies $(2,\lambda)$-TNC. We set the parameter $\lambda=0.001$. For the compared methods, as we could see from Figure \ref{fig1} and \ref{fig3}, the performance of PSA (Algorithm \ref{alg:3}) is quite poor. Moreover, we find that the error of Faster-DPSGD-SC (Algorithm \ref{alg:6}) is also quite large. Thus, to have a better comparison with Phased-SGD-SC, here we will not show the results of PSA  and Faster-DPSGD-SC for the regularized logistic regression problem.

 \subsection*{Dataset and Parameter Settings}
 We will implement our methods on four real-world datasets from the libsvm website\footnote{https://www.csie.ntu.edu.tw/~cjlin/libsvm/}, namely a8a, a9a, ijcnn1 and w7a. For ijcnn1 dataset, although the training set and test set are explicitly provided, the training set is relatively small while test set is relatively too large. Thus, we randomly select $8\times 10^4$ samples in the test data and combine them to the training data, {\em i.e.,} we will leave 11,701 samples for testing while 115,000 samples for training.
 
Since it is difficult to get the exact value of the population risk function, here we will use the testing error to approximate it, which is the value of the empirical risk on test data. In the experimental part, we study the above mentioned two TNC problems (\ref{eq23}), (\ref{eq24}) and their corresponding testing errors with various sample sizes and privacy budgets $\epsilon$. When performing the results for different sample sizes, we will fix $\frac{\epsilon}{2\sqrt{\log(1/\delta)}}=2$. When performing the results for different privacy budgets $\epsilon$, we will use $n=10^4$ samples. We will set $\delta=\frac{1}{n^{1.1}}$ for all experiments. 

Note that all the algorithms presented in the experiments are conducted for 20 random runs and we take the their averaged testing error over the 20 runs. 
  \begin{figure*}[!htbp]
\centering
\subfigure[a8a]{
\includegraphics[width=5.5cm]{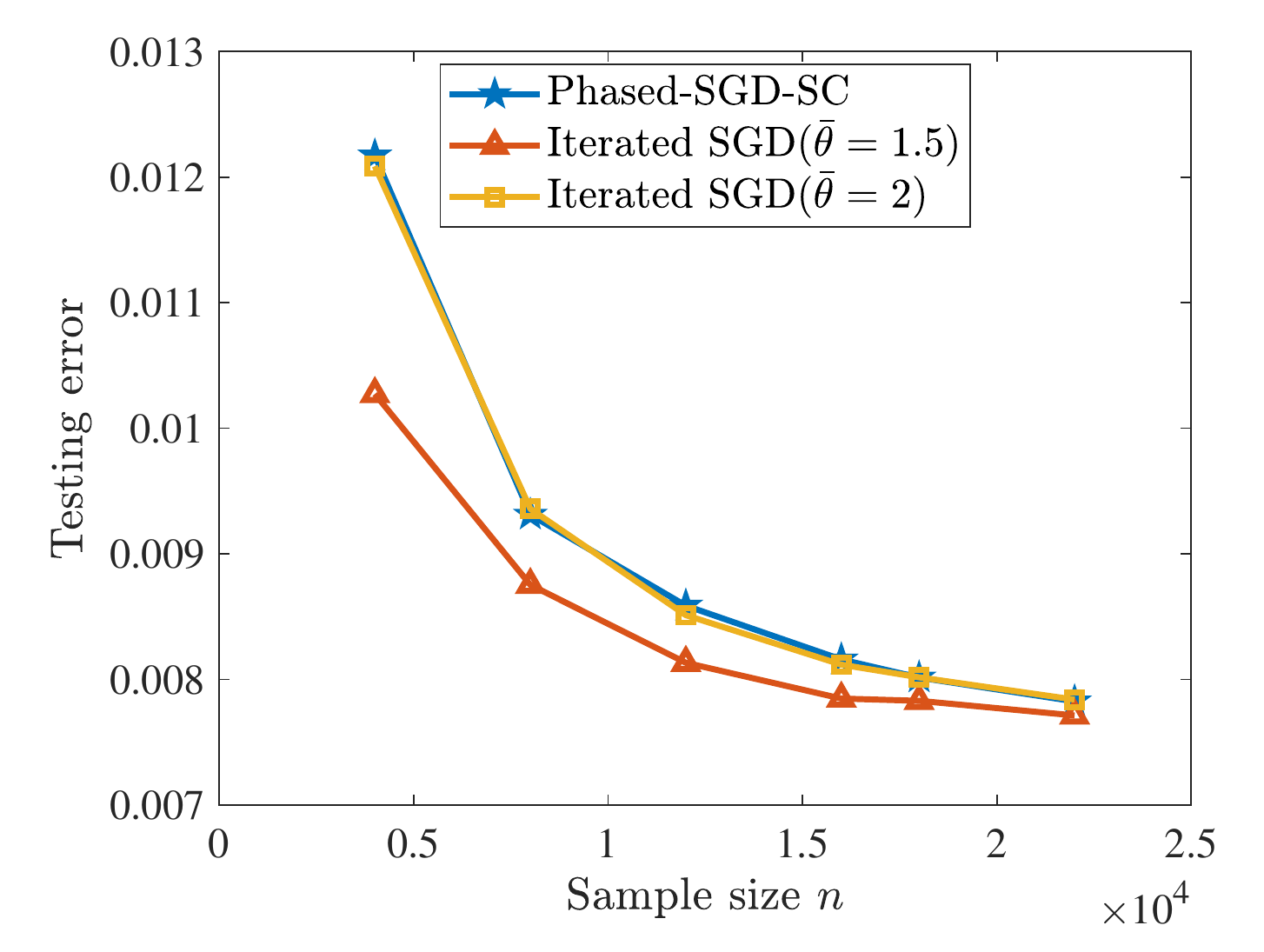}
%\caption{fig1}
}
\quad
\subfigure[a9a]{
\includegraphics[width=5.5cm]{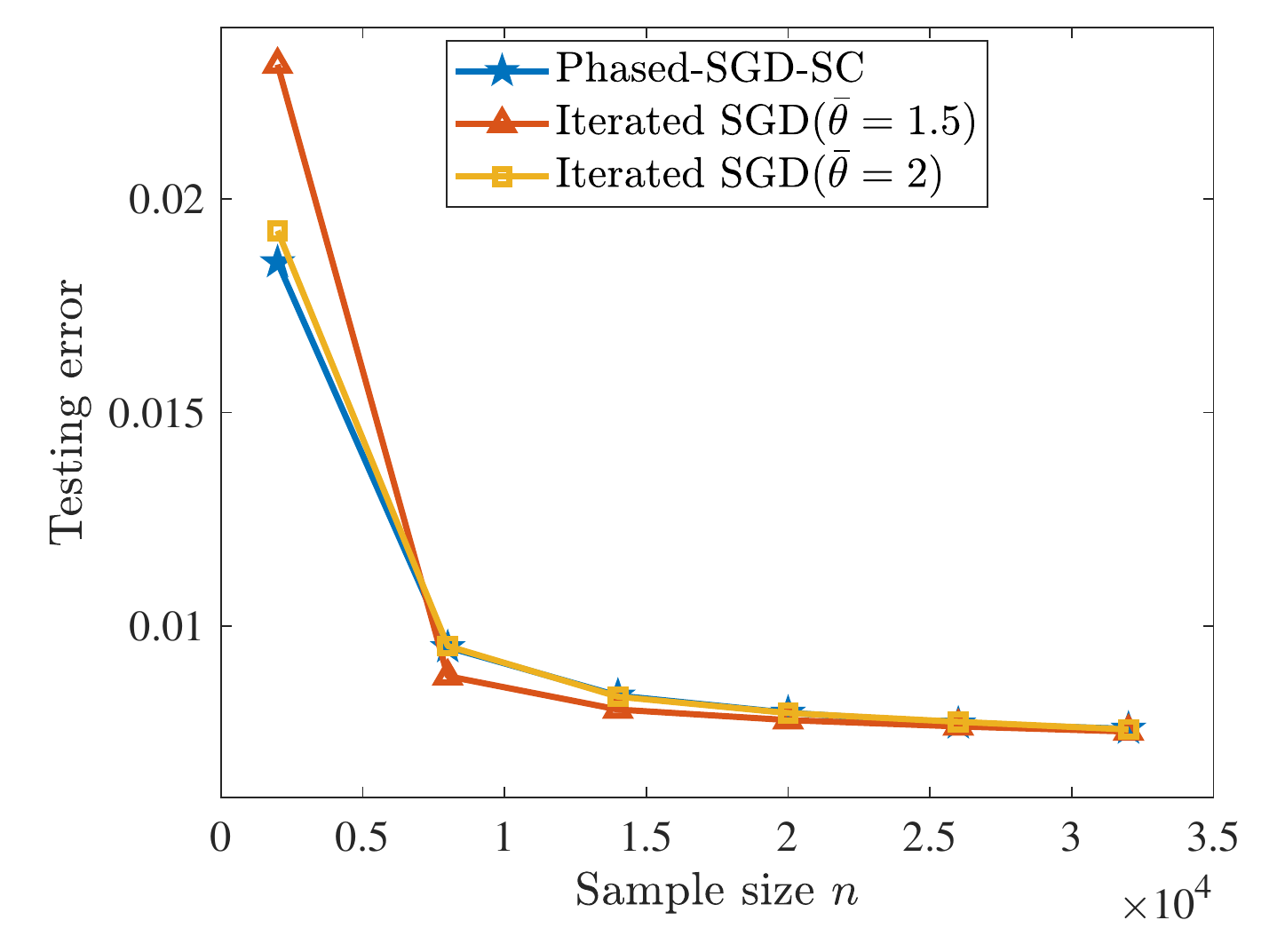}
}
\quad
\subfigure[ijcnn1]{
\includegraphics[width=5.5cm]{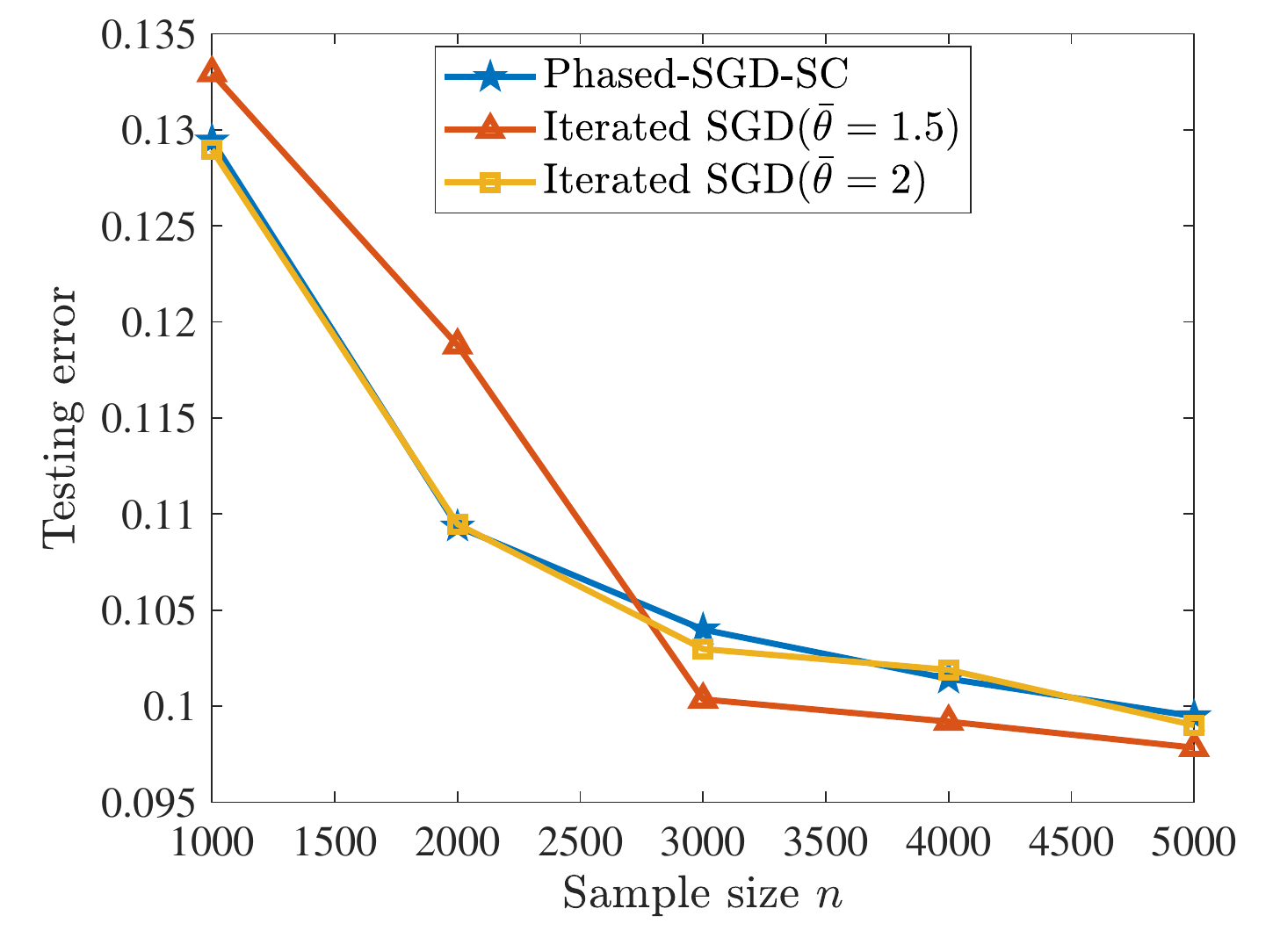}
}
\quad
\subfigure[w7a]{
\includegraphics[width=5.5cm]{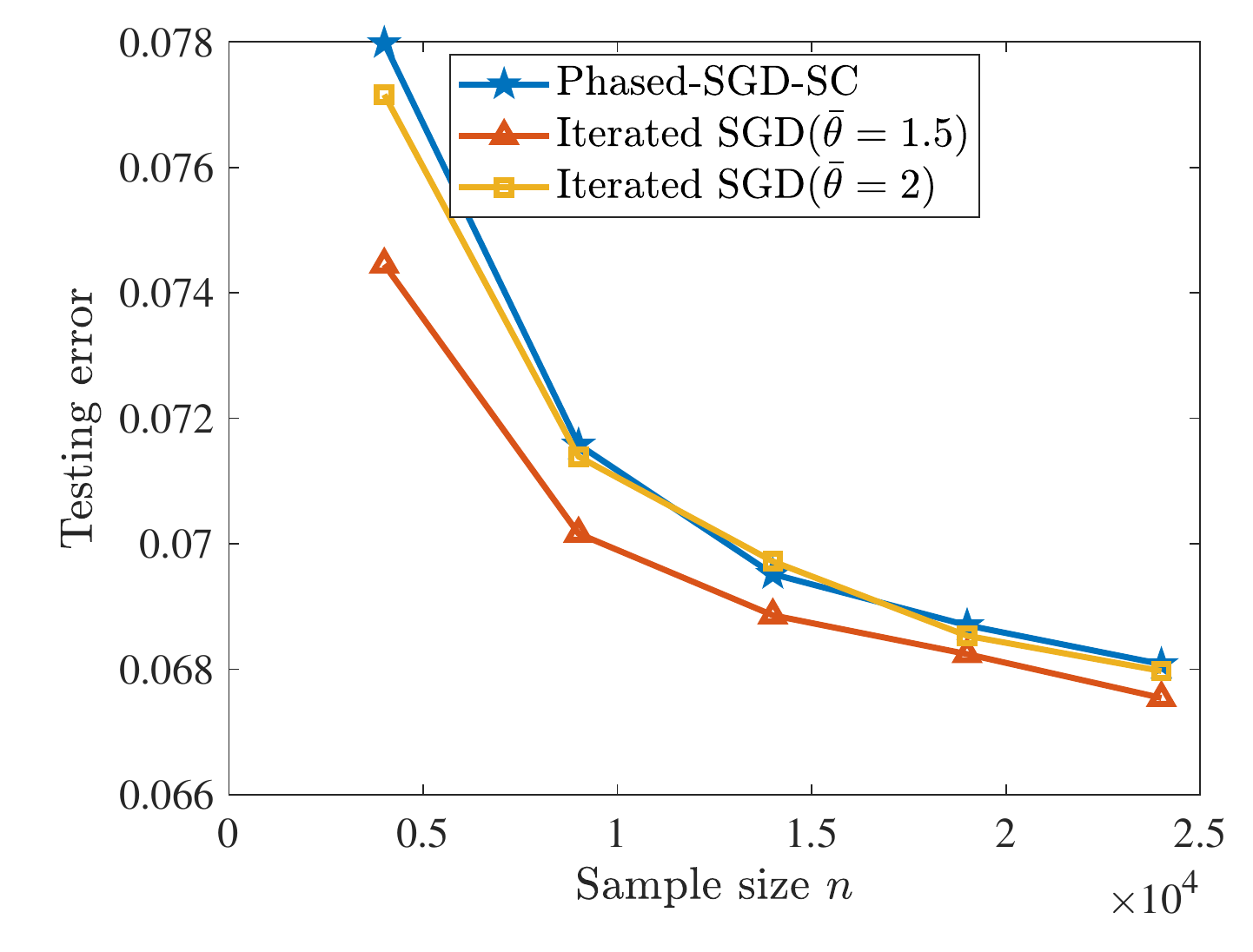}
}

\caption{ Results of $l_2$ regularized squared logistic regression problem (\ref{eq24}) with  different training sample size. \label{fig2}}
\end{figure*}

 \begin{figure*}[!htbp]
\centering
\subfigure[a8a]{
\includegraphics[width=5.5cm]{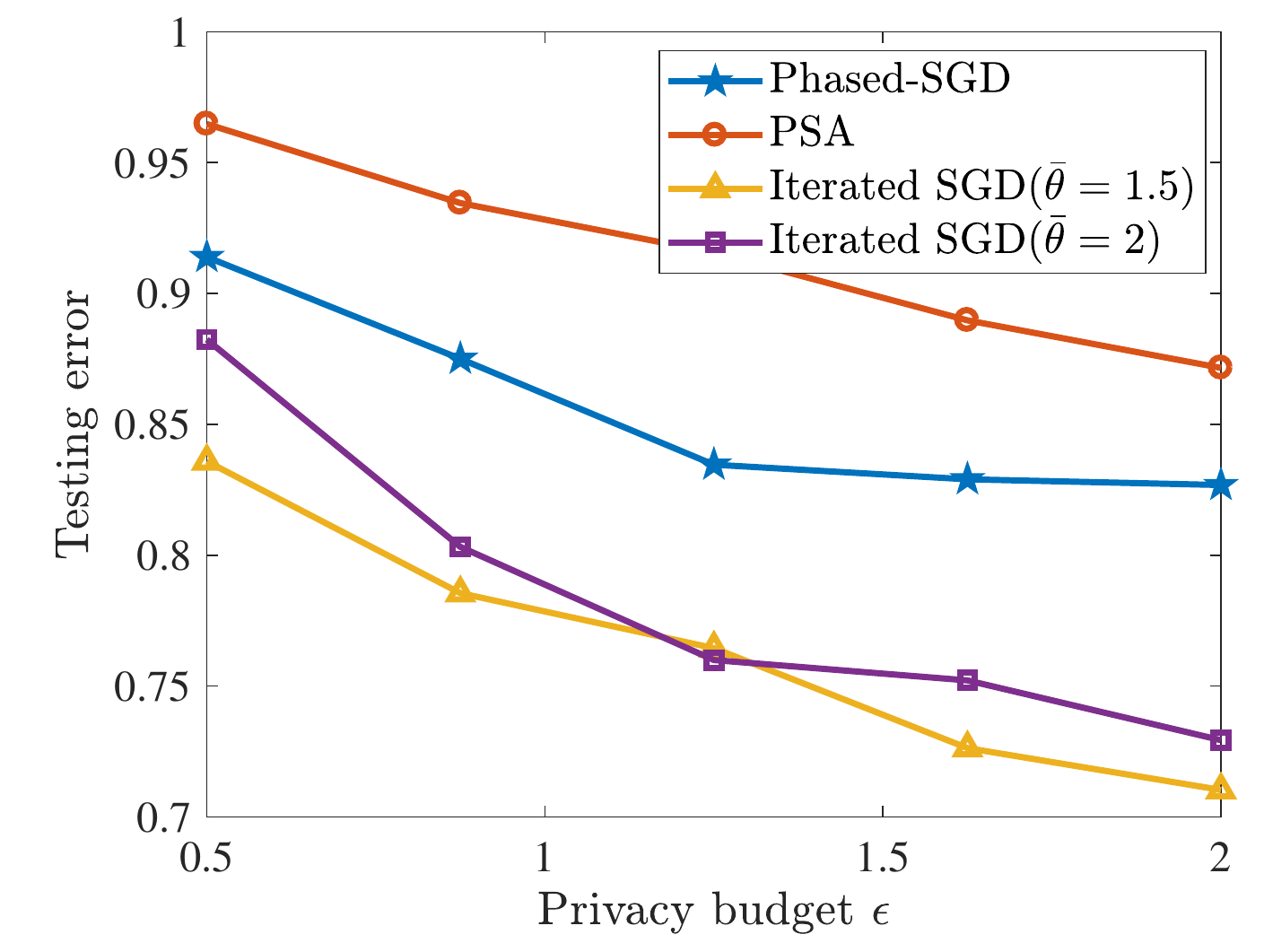}
%\caption{fig1}
}
\quad
\subfigure[a9a]{
\includegraphics[width=5.5cm]{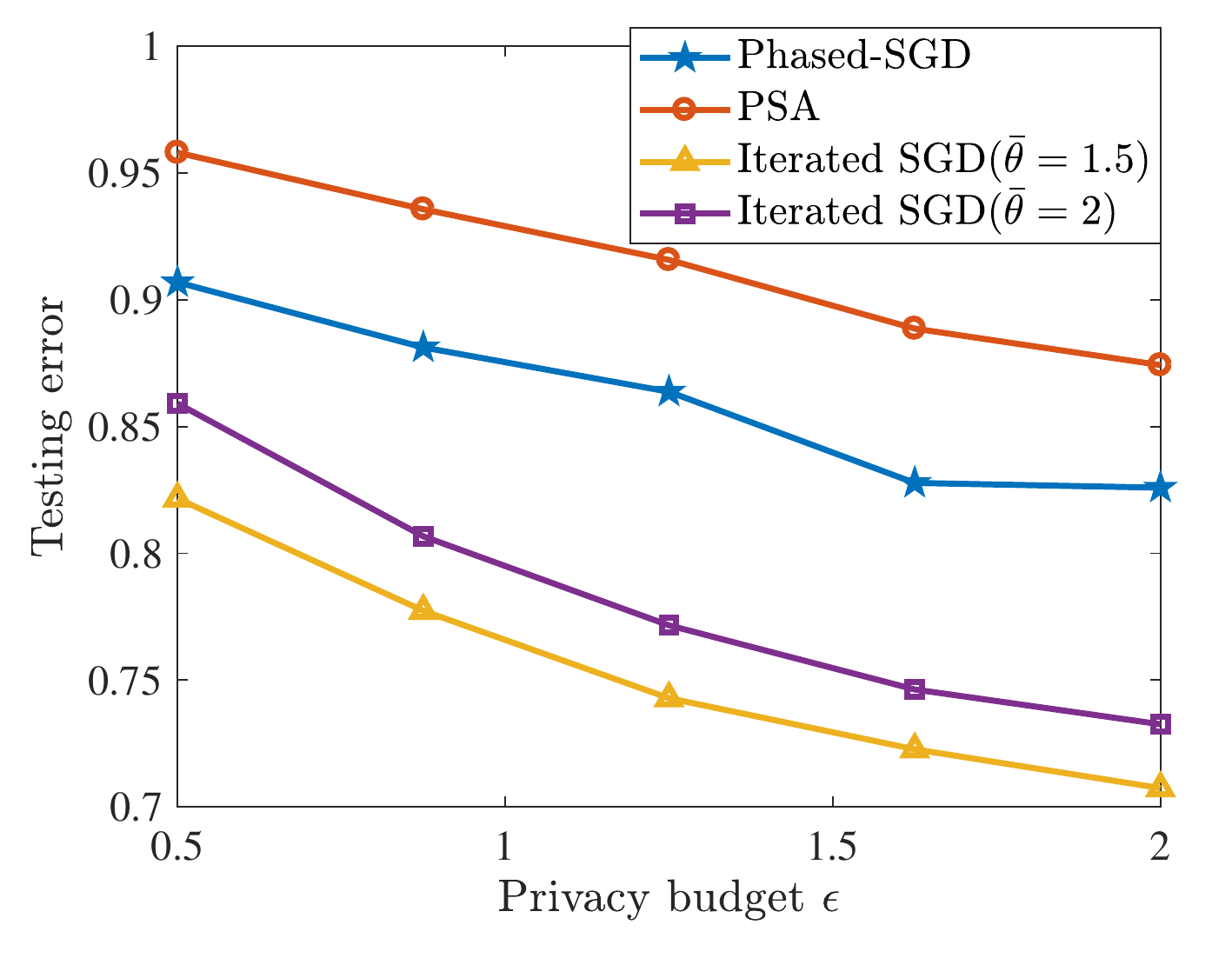}
}
\quad
\subfigure[ijcnn1]{
\includegraphics[width=5.5cm]{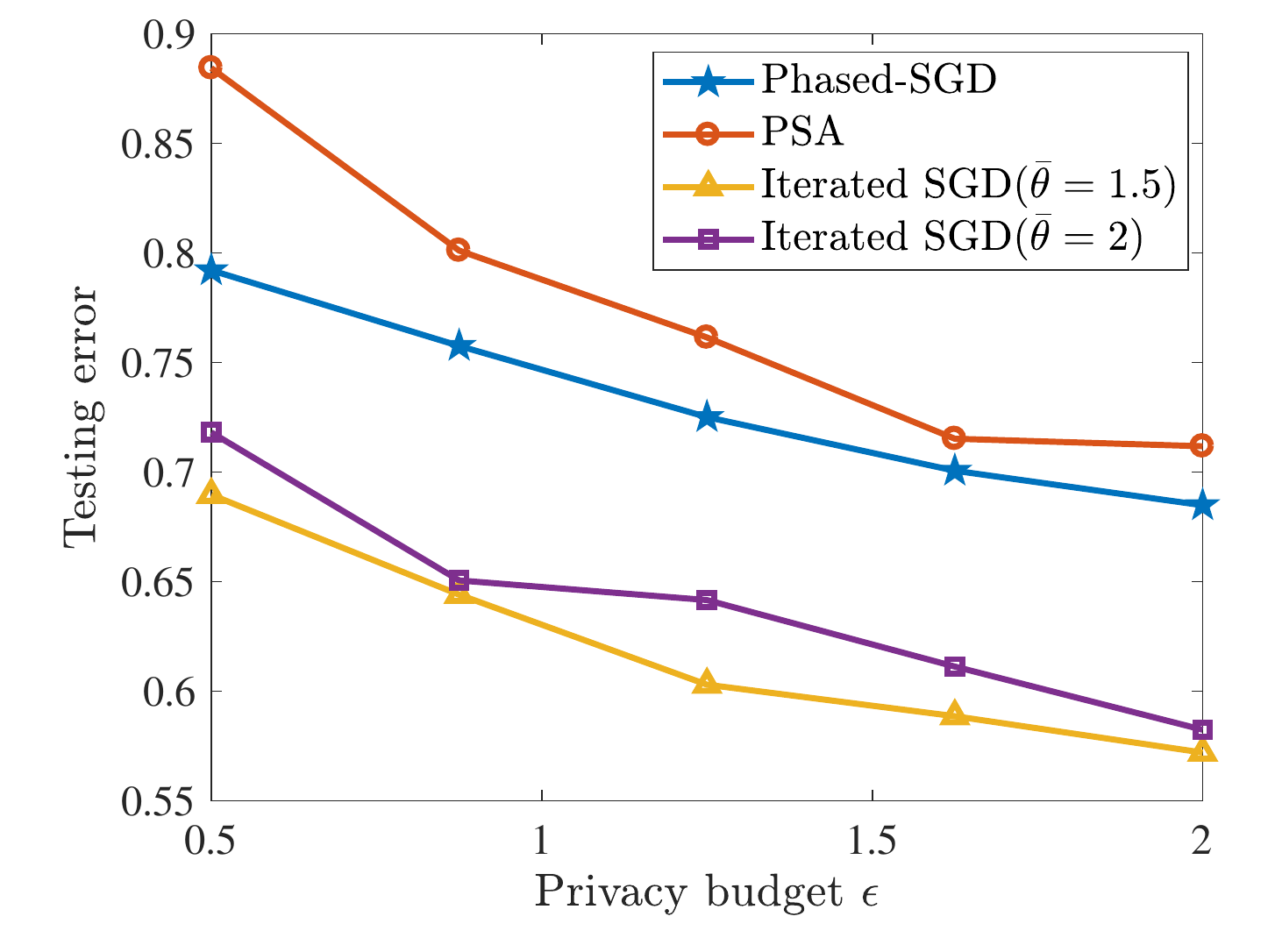}
}
\quad
\subfigure[w7a]{
\includegraphics[width=5.5cm]{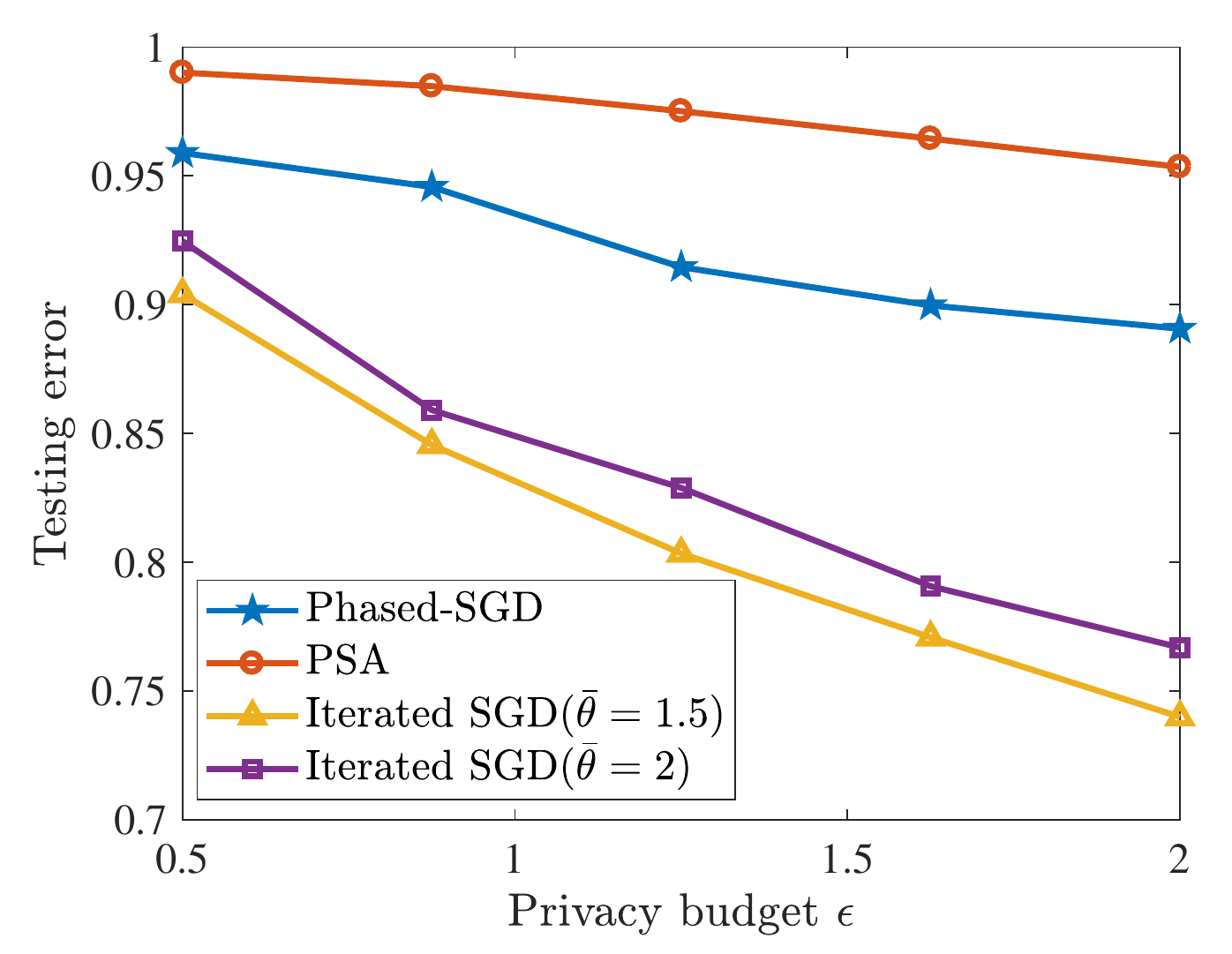}
}
\caption{ Results of Linear regression problem (\ref{eq23}) with  different privacy budget $\epsilon$. }
\label{fig3}
\end{figure*}
\subsection*{Experimental Results}
Figure \ref{fig1} and \ref{fig3} are the results for linear regression, while Figure \ref{fig2} and \ref{fig4} are for the $\ell_2$-norm regularized squared logistic regression. As we can see the from these results: 1) The previous method Phased-SGD outperforms our first method PSA in all the cases, which contradicts to our previous theoretical results. We conjecture that this may be due to that we use the Dykstra’s algorithm in PSA to get an approximate solution of the projection step, which may destroy our theoretical guarantees, another reason may be that the sample size is still not large enough, as we can see when the sample size gets larger the two methods get closer. 2) We can see that, unlike PSA, our third method, Iterated SGD, outperforms Phased-SGD in all the experiments. Moreover, either the sample size or the privacy parameter $\epsilon$ gets larger, the testing error of Iterated SGD decreases faster, which supports our previous theoretical analysis. 3) From Iterated SGD with $\bar{\theta}=2$ and $\bar{\theta}=1.5$ we can see that the method is quite flexible. This is due to that we showed that Theorem \ref{thm:1} will hold as long as  $\theta\geq \bar{\theta}>1$.  However, we note that the performance could still be different for $\bar{\theta}=1.5$ and $\bar{\theta}=2$, and we find that $\bar{\theta}=1.5$ is better than $\bar{\theta}=2$. We conjecture it is because  the hidden constant in the upper bound of Theorem \ref{thm:1} in the case of  $\bar{\theta}=1.5$ is relatively smaller than  the case of  $\bar{\theta}=2$.

 \begin{figure*}[!htbp]
\centering
\subfigure[a8a]{
\includegraphics[width=5.5cm]{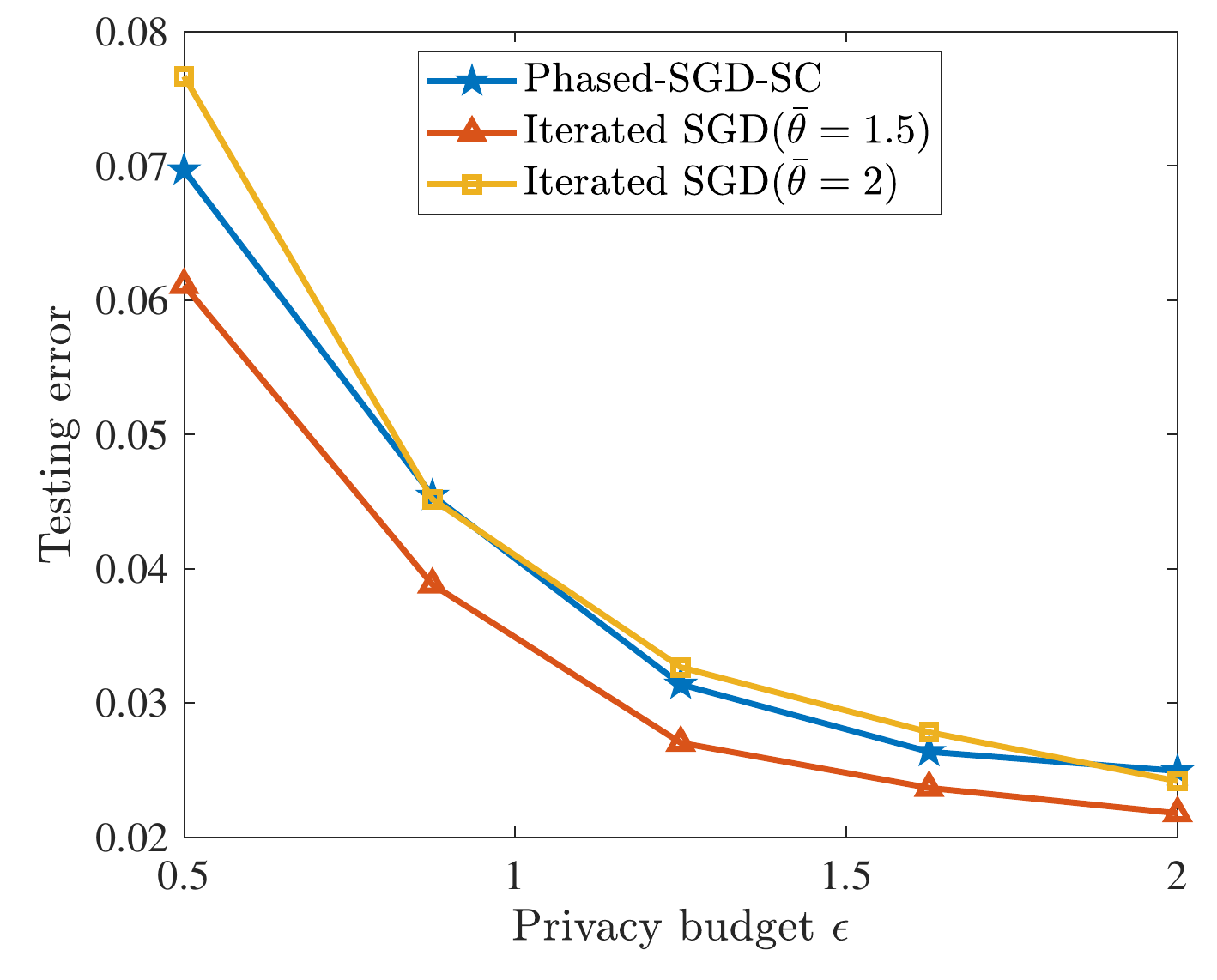}
%\caption{fig1}
}
\quad
\subfigure[a9a]{
\includegraphics[width=5.5cm]{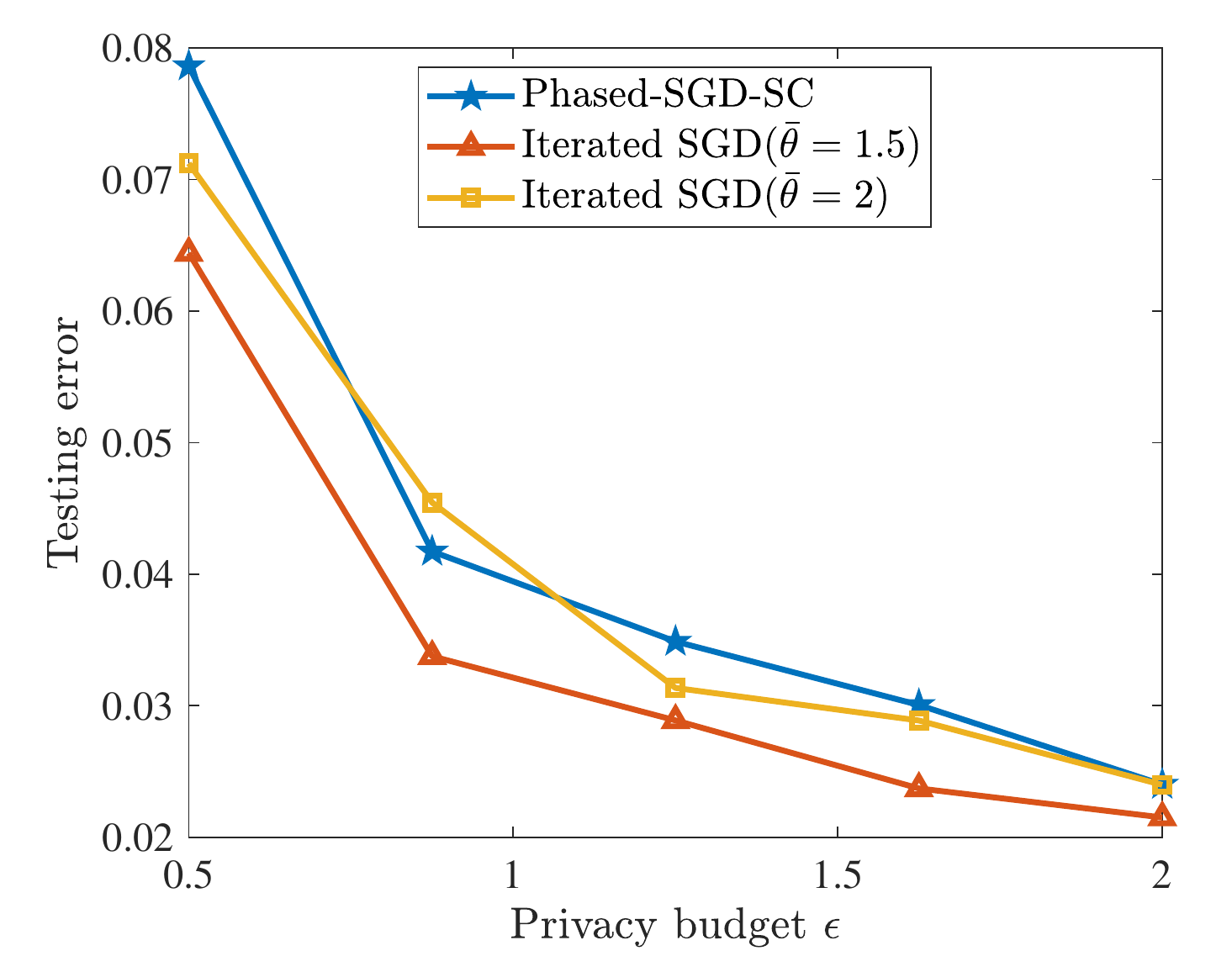}
}
\quad
\subfigure[ijcnn1]{
\includegraphics[width=5.5cm]{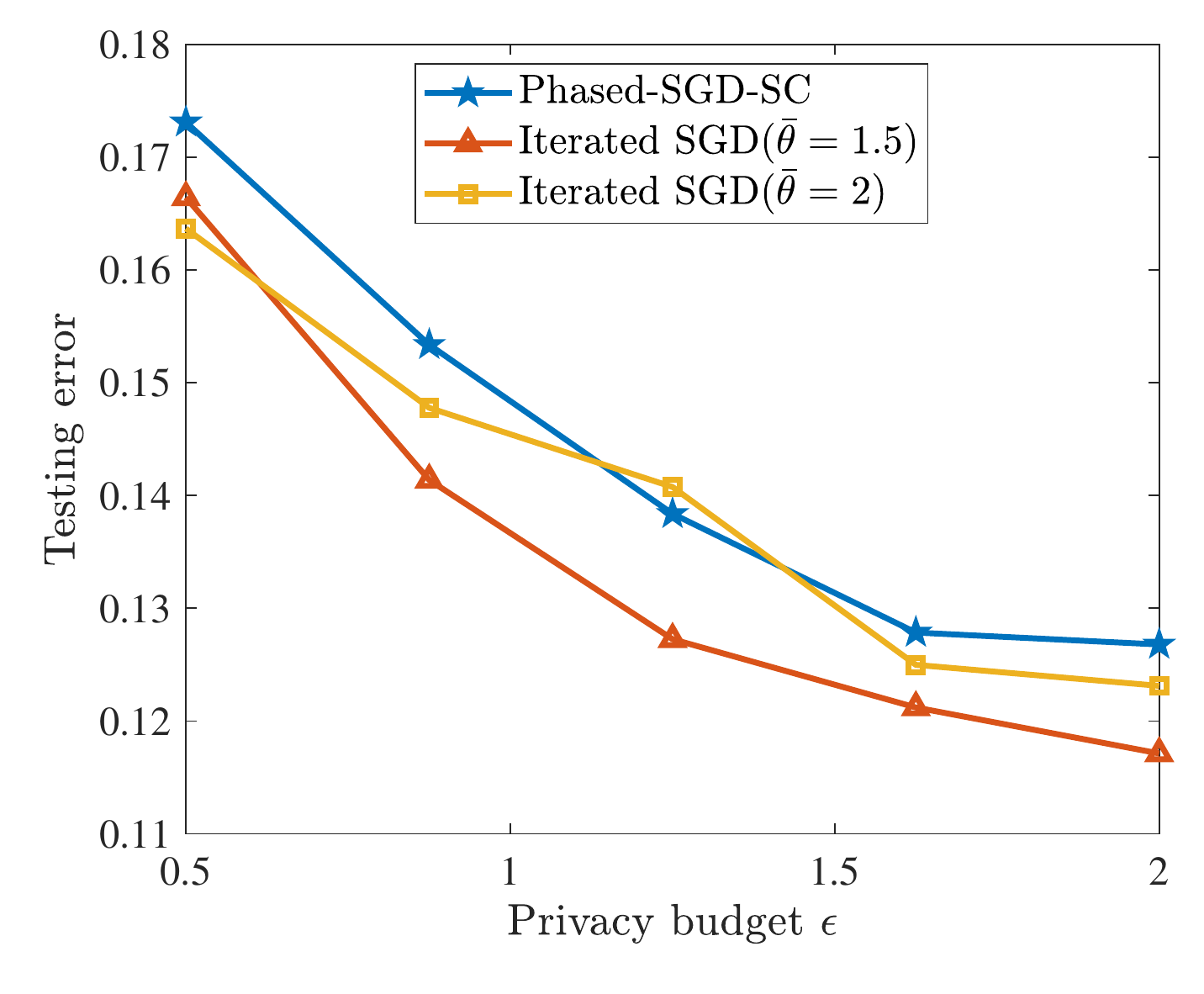}
}
\quad
\subfigure[w7a]{
\includegraphics[width=5.5cm]{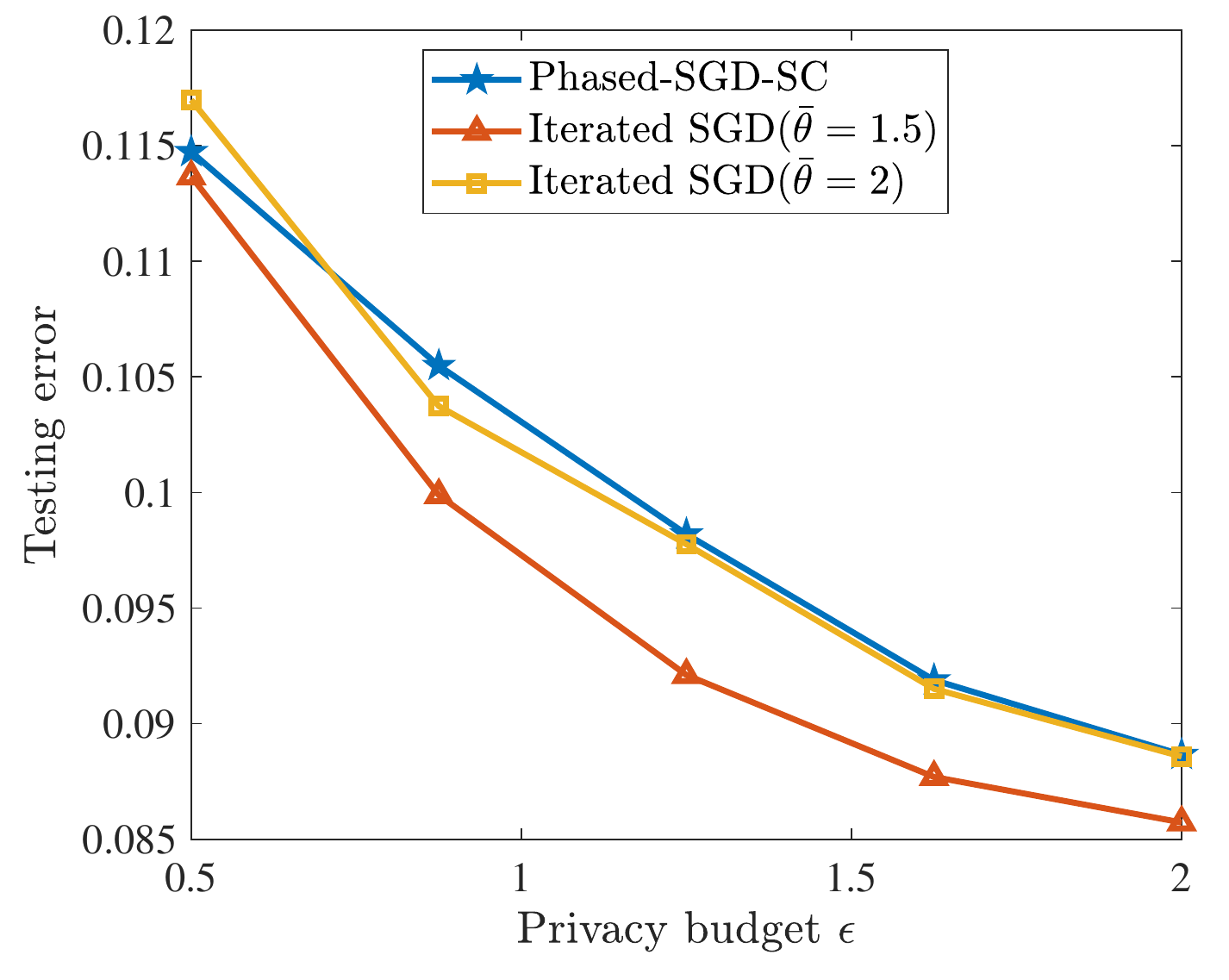}
}
\caption{ Results of $l_2$ regularized squared logistic regression problem (\ref{eq24}) with  different privacy budget $\epsilon$.\label{fig4}}
\end{figure*}

\section{Omitted Proofs}
\begin{proof}[{\bf Proof of Theorem \ref{thm:2}}]
For convenience here we only show the proof of $(\epsilon,\delta)$-DP. The proof of $\epsilon$-DP is almost the same by replacing the term $(\frac{1}{\sqrt{n}}+\frac{\sqrt{d\log(1/\delta)}}{\epsilon n})$ to $(\frac{1}{\sqrt{n}}+\frac{d}{n\epsilon})$ in the following proof. 

		The guarantee of $(\epsilon,\delta)$-DP is just followed by Lemma \ref{lemma:4} and the parallel theorem of Differential Privacy. In the following we will focus on the utility.

For simplicity, we denote $a(n)=10L\left(\frac{1}{\sqrt{n}}+\frac{\sqrt{d\log(1/\delta)}}{\epsilon n}\right)$. We set $\mu_0=2R_0^{1-\theta} a(n_0)$, $\mu_k=2^{(\theta-1)k}\mu_0$ and $R_k=\frac{R_0}{2^k}$, where $k=1,\cdots,m$.

Then we have $\mu_k\cdot R_k^{\theta}=2^{-k}\mu_0 R_0^{\theta}$.
We can also assume that $\lambda \leq \frac{L}{R_0^{\theta-1}}$, otherwise we can set $\lambda= \frac{L}{R_0^{\theta-1}}$, which makes TNC still hold.

Recall that $m=\lfloor \frac{1}{2}\log_2\frac{2n}{\log_2 n}\rfloor -1$, when $n\geq 256$, it follows that 
\begin{equation*}
0< \frac{1}{2}\log_2\frac{2n}{\log_2 n}-2\leq m\leq \frac{1}{2}\log_2\frac{2n}{\log_2 n}-1\leq \frac{1}{2}\log_2 n. 
\end{equation*}
 Thus, we have $2^m\geq \frac{1}{4}\sqrt{\frac{2n}{\log_2 n}}$.

Thus 		\begin{equation}\label{eq22}
\begin{aligned}
\mu_m=2^{(\theta-1)m} \mu_0&\geq 2^m\mu_0\\
&\geq \frac{1}{4}\sqrt{\frac{2n}{\log_2 n}}\cdot 2\cdot R_0^{1-\theta}a(n_0)\\
&=5\cdot L R_0^{1-\theta}\sqrt{\frac{2n}{\log_2 n}}\left(\frac{1}{\sqrt{\frac{n}{m}}}+\frac{\sqrt{d\log(1/\delta)}}{\epsilon\cdot\frac{n}{m}}\right)\\
&\geq 5\cdot L R_0^{1-\theta}\sqrt{\frac{2n}{\log_2 n}}\left(\frac{1}{\sqrt{\frac{2n}{\log_2 2n-\log_2\log_2 n-4}}}\right)\\
&=5\cdot L R_0^{1-\theta}\sqrt{\frac{\log_2 2n-\log\log_2 n-4}{\log_2 n}}\\&\geq L R_0^{1-\theta}~\left(\text{Since}~5\cdot\sqrt{\frac{\log_2 2n-\log\log_2 n-4}{\log_2 n}}\geq 1~ \text{when}~ n\geq 256\right)\\
&\geq \lambda~(\text{By assumption}).
\end{aligned}
\end{equation}
where the third inequality is given by throwing away the $\frac{\sqrt{d\log(1/\delta)}}{\epsilon\cdot\frac{n}{m}}$ term and substituting $ m$ in term $\frac{1}{\sqrt{\frac{n}{m}}}$ with $\frac{1}{2}\log_2\frac{2n}{\log_2 n}-2$.

Below, we consider the following two cases.
\begin{case}
	If $\lambda\geq \mu_0$, then $\mu_0\leq \lambda\leq \mu_m$. We have the following lemma.
	\begin{lemma}\label{lemma:12}
		Let $k^{*}$ satisfies $\mu_{k^{*}}\leq \lambda \leq 2^{\theta -1}\mu_{k^{*}}$, then for any $1\leq k\leq k^{*}$, the points $\{\hat{w}_k\}_{k=1}^{m}$ generated by Algorithm \ref{alg:3} satisfy
		\begin{equation}\label{eq15}
		\mathbb{E}[||\hat{w}_{k-1}-w^{*}||_2]\leq R_{k-1}=2^{-(k-1)}\cdot R_0,
		\end{equation}
		
			\begin{equation}\label{eq16}
		\mathbb{E}[F(\hat{w}_k)]-F(w^{*})\leq \mu_k R_k^{\theta} =2^{-k}\mu_0 R_0^{\theta}.
		\end{equation}
		Moreover, for $k\geq k^{*}$, we have 
			\begin{equation}\label{eq17}
		\mathbb{E}[F(\hat{w}_k)]-\mathbb{E}[F(\hat{w}_{k^{*}})]\leq \mu_{k^{*}}R_{k^{*}}^{\theta}.
		\end{equation}
	\end{lemma}
\begin{proof}[Proof of Lemma \ref{lemma:12}]
We prove (\ref{eq15}), (\ref{eq16}) by induction. Note that (\ref{eq15}) holds for $k=1$.

Assume (\ref{eq15}) is true for some $k>1$, then we have
	\begin{equation}
\begin{aligned}
\mathbb{E}[F(\hat{w}_k)]-F(w^{*})&\leq 10 R_{k-1}\cdot L\left(\frac{1}{\sqrt{n_0}}+\frac{\sqrt{d\log(1/\delta)}}{\epsilon\cdot n_0}\right)\\
&=R_{k-1} a(n_0)\\
&=\frac{1}{2}\mu_k 2^{(1-\theta)k} R_0^{\theta-1} R_{k-1}\\
&=\mu_k R_k^{\theta}
\end{aligned}
\end{equation}
Which is (\ref{eq16}). By the definition of TNC,  we have 
	\begin{equation}
\begin{aligned}
\mathbb{E}||\hat{w}_k-w^{*}||_2^{\theta}&\leq \frac{1}{\lambda} (\mathbb{E}[F(\hat{w}_k)]-F(w^{*}))\\
&\leq \frac{\mathbb{E}[F(\hat{w}_k)]-F(w^{*})}{\mu_{k^{*}}}\\&\leq \frac{\mu_k R_k^{\theta}}{\mu_{k^{*}}}\leq R_k^{\theta}
\end{aligned}
\end{equation}
Thus (\ref{eq15}) is true for $k+1$.

Now we prove (\ref{eq17}). 
Referring to Lemma \ref{lemma:4}, we know that 
	\begin{equation*}
\begin{aligned}
\mathbb{E}[F(\hat{w}_k)]-\mathbb{E}[F(\hat{w}_{k-1})]
&\leq R_{k-1} \cdot a(n_0)\\
&=2^{k^{*}-k} R_{k^{*}-1} a(n_0)\\
&=2^{k^{*}-k} \mu_{k^{*}}R_{k^{*}}^{\theta}\\
&=\mu_k R_k^{\theta}
\end{aligned}
\end{equation*}
Thus, for $k>k^{*}$,
	\begin{equation*}
\begin{aligned}
\mathbb{E}[F(\hat{w}_k)]-\mathbb{E}[F(\hat{w}_{k^{*}})]
&=\sum\limits_{j=k^{*}+1}^{k}
(\mathbb{E}[F(\hat{w}_j)]-\mathbb{E}[F(\hat{w}_{j-1})])\\
&\leq \sum\limits_{j=k^{*}+1}^{k}2^{k^{*}-j}\mu_{k^{*}} R_{k^{*}}^{\theta}\\
&=(1-2^{k^{*}-k})\mu_{k^{*}}R_{k^{*}}^{\theta}\\
&\leq \mu_{k^{*}}R_{k^{*}}^{\theta}
\end{aligned}
\end{equation*}
Here completes the proof of the lemma.
\end{proof}

Now we proceed to prove theorem \ref{thm:2} in this case.
	\begin{equation}
\begin{aligned}
\mathbb{E}[F(\hat{w}_m)]-F(w^{*})&=(\mathbb{E}[F(\hat{w}_m)]-\mathbb{E}[F(\hat{w}_{k^{*}})])+(\mathbb{E}[F(\hat{w}_{k^{*}})]-F(w^{*}))\\
&\leq 2\mu_{k^{*}} R_{k^{*}}^{\theta}\\
&\leq 4\left(\frac{\mu_{k^{*}}}{\lambda}\right)^{\frac{1}{\theta-1}}\mu_{k^{*}}R_{k^{*}}^{\theta}~(\text{Since} \left(\frac{\mu_{k^{*}}}{\lambda}\right)^{\frac{1}{\theta-1}}\geq \frac{1}{2})\\
&=4\left(\frac{2^{(\theta-1)k^{*}}\mu_0}{\lambda}\right)^{\frac{1}{\theta-1}}\mu_{k^{*}}R_{k^{*}}^{\theta}\\
&=4(2^{k^{*}}\mu_{k^{*}} R_{k^{*}}^{\theta}\mu_0^{\frac{1}{\theta -1}}\left(\frac{1}{\lambda}\right)^{\frac{1}{\theta-1}}  )\\
&=4(\mu_0 R_0^{\theta}\mu_0^{\frac{1}{\theta -1}}\left(\frac{1}{\lambda}\right)^{\frac{1}{\theta-1}})\\
&=4( R_0^{\theta}\mu_0^{\frac{\theta}{\theta -1}}\left(\frac{1}{\lambda}\right)^{\frac{1}{\theta-1}})\\
&=4\cdot((2\cdot a(n_0))^{\frac{\theta}{\theta -1}}\left(\frac{1}{\lambda}\right)^{\frac{1}{\theta-1}})
\\&=4\cdot\left(\frac{1}{\lambda}\right)^{\frac{1}{\theta-1}}\cdot\left(20L\left(\frac{\sqrt{m}}{\sqrt{n}}+\frac{m}{n}\frac{\sqrt{d\log(1/\delta)}}{\epsilon}\right)\right)^{\frac{\theta}{\theta -1}}
\end{aligned}
\end{equation}
where $m=O(\log_2 n)$.(Recall that $m\leq \frac{1}{2}\log_2 n$).
\end{case}
\begin{case}
If $\lambda<\mu_0$, then 
	\begin{equation*}
\begin{aligned}
\mathbb{E}[F(\hat{w}_1)]-F(w^{*})&\leq R_0a(n_0)\\&=\left(\frac{2}{\mu_0}\right)^{\frac{1}{\theta-1}}\cdot a(n_0)^{\frac{\theta}{\theta-1}}\\&<\left(\frac{2}{\lambda}\right)^{\frac{1}{\theta-1}}\cdot a(n_0)^{\frac{\theta}{\theta-1}}
\end{aligned}
\end{equation*}
Also, we have 
	\begin{equation*}
\begin{aligned}
\mathbb{E}[F(\hat{w}_m)]-\mathbb{E}[F(\hat{w}_1)]
&=\sum\limits_{j=2}^{m}
(\mathbb{E}[F(\hat{w}_j)]-\mathbb{E}[F(\hat{w}_{j-1})])\\
&\leq \sum\limits_{j=2}^{m}R_{j-1}\cdot a(n_0)\\
&=\sum\limits_{j=2}^{m}2^{-(j-1)}R_0\cdot a(n_0)\\&=(1-(1/2)^{m-1}) R_0 \cdot a(n_0)<R_0\cdot  a(n_0)
\end{aligned}
\end{equation*}

By a similar argument process as in \textbf{Case 1}, we have 
	\begin{equation}
\begin{aligned}
\mathbb{E}[F(\hat{w}_m)]-F(w^{*})&=(\mathbb{E}[F(\hat{w}_m)]-\mathbb{E}[F(\hat{w}_1)])+(\mathbb{E}[F(\hat{w}_1)]-F(w^{*}))\\
&\leq 2R_0 a(n_0)\leq 2\left(\frac{2}{\lambda}\right)^{\frac{1}{\theta-1}}\cdot a(n_0)^{\frac{\theta}{\theta-1}}
\\&=2\cdot\left(\frac{2}{\lambda}\right)^{\frac{1}{\theta-1}}\cdot\left(10L\left(\frac{\sqrt{m}}{\sqrt{n}}+\frac{m}{n}\frac{\sqrt{d\log(1/\delta)}}{\epsilon}\right)\right)^{\frac{\theta}{\theta -1}}
\end{aligned}
\end{equation}
\end{case}
Combining the two cases, we conclude that 
	\begin{equation*}
\begin{aligned}
\mathbb{E}[F(\hat{w}_m)]-F(w^*)\leq O\left(\left(\frac{L^\theta}{\lambda}\right)^{\frac{1}{\theta-1}}\cdot\left(\frac{\sqrt{\log n}}{\sqrt{n}}+\frac{\sqrt{d\log(1/\delta)}\cdot\log n}{n\epsilon}\right)^{\frac{\theta}{\theta-1}}\right)
\end{aligned}
\end{equation*}

\end{proof}
	\begin{proof}[ {\bf Proof of Theorem \ref{new:th1}}]
	 	  Before our proof, we provide some notations. We denote $F^*=\min_{w\in\mathcal{W}}F(w)$. For a given error $\rho$, we denote $\mathcal{L}_\rho$ the $\rho$-level set of  function $F(W)$ and $\mathcal{S}_\rho$  the $\rho$-sublevel set $F(w)$, respectively, {\em i.e.,} $\mathcal{L}_\rho=\{w\in \mathcal{W}: F(w)=F^*+\rho\}$, $\mathcal{S}_\rho=\{w\in\mathcal{W}: F(w)\leq F^*+\rho\}$. For any $w\in \mathcal{W}$, we denote $w^{+}_\rho$ as the closet point in the $\rho$-sublevel set to $w$, {\em i.e.,} 
	  \begin{equation*}
	      w^{+}_\rho=\arg\min_{v\in \mathcal{S}_\rho}\|v-w\|_2^2. 
	  \end{equation*}
	    Using the KKT condition, it is easy to check that when $w\not\in \mathcal{S}_\rho$ then $w^{+}_\rho \in \mathcal{L}_\rho$. We first recall the following lemma, given by \citep{yang2018rsg}. 
	    \begin{lemma}[Lemma 1 in \citep{yang2018rsg}]\label{lemma:15}
	     For any $w\in\mathcal{W}$ and $\rho>0$ we have 
	     \begin{equation*}
	         \|w-w^{+}_\rho\|_2\leq \frac{\text{dist}(w^{+}_\rho, \mathcal{W}_*) }{\rho} (F(w)-F(w^+_\rho)),
	     \end{equation*}
	     	    where $\mathcal{W}_*=\{w: w\in \arg\min_{w\in \mathcal{W}}F(w)\}$ and $\text{dist}(w^{+}_\rho, \mathcal{W}_*)$ is the distance from the point $w^+_\rho$ to the set $ \mathcal{W}_*$. 
	    \end{lemma}

	    \begin{lemma}\label{lemma:14}
	    If $f(\cdot, x)$ is  convex, $\beta$-smooth and $L$-Lipschitz for each $x$ and $\gamma \geq \frac{\|\mathcal{W}\|_2}{L}$, where $\|\mathcal{W}\|_2$ is the diameter of the set $\mathcal{W}$, {\em i.e.,} $\|\mathcal{W}\|_2=\max_{w, w'\in \mathcal{W}}\|w-w'\|_2$.  
	    Based on different noises and stepsizes in Algorithm \ref{alg:new},  Algorithm \ref{alg:new} is $(\epsilon, \delta)$ or $\epsilon$-DP if $\eta\leq \frac{1}{\beta}$. 
	   Given $w_0\in \mathcal{W}$, for the output $w_k$ in Algorithm \ref{alg:new}. In the case of $(\epsilon, \delta)$-DP, we have 
	    \begin{equation*}
	\mathbb{E}[\hat{F}(w_k)]-\min_{w\in \mathcal{W}} \hat{F}(w)\leq 3200L^2\gamma (\frac{1}{n}+\frac{d\log(1/\delta)}{n^2\epsilon^2 }). 
	    \end{equation*}
	  In the case of $\epsilon$-DP, we have 
	  	    \begin{equation*}
	\mathbb{E}[\hat{F}(w_k)]-\min_{w\in \mathcal{W}} \hat{F}(w)\leq 3200L^2\gamma (\frac{1}{n}+\frac{d^2}{n^2\epsilon^2 }), 
	    \end{equation*}
	    where $\hat{F}(w)=F(w)+\frac{1}{2\gamma}||w-w_0||_2^2$ and $w_0$ is the initial point.
	    \end{lemma}
	    \begin{proof}[{\bf Proof of Lemma \ref{lemma:14}}]
	    We can see the regularized  function of $\hat{F}(w)$ as a population risk with loss function $\tilde{f}(w, x)= f(w, x)+\frac{1}{2\gamma}\|w-w_0\|_2^2$. Thus, by the assumption of $f(\cdot, x)$, we have $\tilde{f}(\cdot, x)$ is $L+\frac{\|\mathcal{W}\|_2}{\gamma}\leq 2L$-Lipschitz, $\beta+\frac{1}{\gamma}$-smooth and $\frac{1}{\gamma}$-strongly convex. Thus, by Theorem 5.1 in \citep{feldman2020private} we have the results. 
	    
	    \end{proof}
	 Next we start our proof.  For convenience here we only focus on $(\epsilon, \delta)$-DP, the proof of $\epsilon$-DP is almost the same. The guarantee of $(\epsilon, \delta)$-DP is simply followed by Lemma \ref{lemma:14}.  We also note that Lemma \ref{lemma:14} implies that for any $w\in \mathcal{W}$,
	  \begin{equation}\label{eq21}
	\mathbb{E}[{F}(w_k)]- {F}(w)\leq \frac{1}{2\gamma}||w-w_0||_2^2+3200L^2\gamma (\frac{1}{n}+\frac{d\log(1/\delta)}{n^2\epsilon^2 }). 
	    \end{equation}
    We denote $\rho=(\frac{8\times 3200 L^2}{\lambda^\frac{2}{\theta}} (\frac{1}{n_0}+\frac{d\log(1/\delta)}{n_0^2\epsilon^2}))^\frac{\theta}{2(\theta-1)}$, $\chi_k=\frac{\chi_0}{2^k}$ and $\gamma_k=\frac{\gamma_0}{2^k}$. Then we have 
    \begin{equation}\label{eq:22}
        \frac{1}{\gamma_0}=\frac{\lambda^\frac{2}{\theta}}{4 \chi_0} \rho^ \frac{2(\theta-1)}{\theta}=\frac{2^{k-2}\lambda^\frac{2}{\theta}}{ \chi_k} \rho^ \frac{2(\theta-1)}{\theta} . 
    \end{equation}
    
     We assume that for all $i\in \{0, 1 \cdots, m-1\}$, $\mathbb{E}[F(w_i)]-F^* > 2\rho$. Otherwise we have proved the theorem.

	   We will show by induction that 
	   \begin{equation}\label{eq:23}
	   \mathbb{E}[F(w_k)]-F^*\leq \chi_k+\rho.
	   \end{equation}
	   If this is true then when $w=m$ we have 
	   \begin{equation*}
	      \mathbb{E}[F(w_k)]-F^*\leq O\left(  (\frac{ L^2}{\lambda^\frac{2}{\theta}} (\frac{1}{n_0}+\frac{d\log(1/\delta)}{n_0^2\epsilon^2}))^\frac{\theta}{2(\theta-1)}\right).  
	   \end{equation*}
	   In the following we will show (\ref{eq:23}). For $k=0$, by the definition of $\chi$, it is true. Now, consider the $k$-th phase. By (\ref{eq21}) we have 
	   \begin{equation*}
	       \mathbb{E}[F(w_k)-F(w^+_{k-1, \rho})]\leq \underbrace{\frac{1}{2\gamma_k}\mathbb{E}\|w^+_{k-1, \rho}-w_{k-1}\|_2^2}_{A}+ \underbrace{3200L^2\gamma_k (\frac{1}{n_0}+\frac{d\log(1/\delta)}{n_0^2\epsilon^2 })}_{B}. 
	   \end{equation*}
	   Since $w_{k-1}\not\in \mathcal{S}_\rho$, $w^+_{k-1, \rho} \in \mathcal{L}_\rho$. Moreover, since we have $\mathbb{E} [F(w_{k-1})]-F(w^*)\leq  \chi_{k-1}+\rho$, we have $\mathbb{E} [F(w_{k-1})]-\mathbb{E} [F^+(w_{k-1, \rho})]\leq \chi_k$. For term $A$, by Lemma \ref{lemma:15} we have 
	   \begin{equation*}
	      \mathbb{E} \|w^+_{k-1, \rho}-w_{k-1}\|_2 \leq \frac{1}{\lambda^\frac{1}{\theta}\rho^{1-\frac{1}{\theta}}}\chi_{k-1}.  
	   \end{equation*}
	   Thus, 
	   \begin{align*}
	       \frac{1}{2\gamma_k}\mathbb{E}\|w^+_{k-1, \rho}-w_{k-1}\|_2^2\leq   \frac{1}{2\gamma_k} (\frac{1}{\lambda^\frac{2}{\theta}\rho^{\frac{2(\theta-1)}{\theta}}}\chi^2_{k-1})= \frac{\chi_{k-1}}{4}, 
	   \end{align*}
	  where the last equality is due to (\ref{eq:22}). 
	  
	  For term $B$, we have 
	  \begin{align*}
	    3200L^2\gamma_k (\frac{1}{n_0}+\frac{d\log(1/\delta)}{n_0^2\epsilon^2 })=  3200L^2\frac{4\chi_0}{2^k\lambda^\frac{2}{\theta}\rho^\frac{2(\theta-1)}{\theta} }(\frac{1}{n_0}+\frac{d\log(1/\delta)}{n_0^2\epsilon^2 })=\frac{\chi_0}{4\times 2^{k-1}}=\frac{\chi_{k-1}}{4},
	  \end{align*}
	  where the first equality is due to (\ref{eq:22}). Thus, in total we have 
	  \begin{equation*}
	         \mathbb{E}[F(w_k)-F(w^+_{k-1, \rho})] \leq \frac{\chi_{k-1}}{2}=\chi_k. 
	  \end{equation*}
	  That is $ \mathbb{E}[F(w_k)]-F^* \leq \chi_k+\rho$. 
	\end{proof}
	\begin{proof}[{\bf Proof of Theorem \ref{thm:1}}]
		In the following we only consider the $(\epsilon, \delta)$-DP case. It is almost the same for $\epsilon$-DP. 
		
		The guarantee of $(\epsilon,\delta)$-DP is just followed by Lemma \ref{lemma:4} and the parallel theorem of Differential Privacy. In the following we will focus on the utility. 
		
		Since $k=\lfloor (\log_{\bar{\theta}}2)\cdot \log\log n\rfloor$, then $k\leq (\log_{\bar{\theta}}2)\cdot \log\log n$, namely $2^k\leq (\log n)^{\log_{\bar{\theta}}2}$ and $\frac{2^k-1}{(\log n)^{\log_{\bar{\theta}}2}}\leq 1$.
		Observe that the total sample number used in the algorithm is $\sum_{i=1}^{k} n_i \leq  \sum_{i=1}^{k}\frac{2^{i-1} n}{(\log n)^{\log _{\bar{\theta}} 2}}=\frac{(2^k-1) n}{(\log n)^{\log _{\bar{\theta}} 2}}\leq n$.
		
		For the output of phase $i$, denote $\Delta_i=\mathbb{E}[F(w_i)]-F(w^*)$, and let $D_i^{\theta}=\mathbb{E}[||w_i-w^*||_2^{\theta}]$.
		The assumption of TNC implies that $F(w_i)- F(w^*)\geq \lambda||w_i-w^*||_2^{\theta}$, which will be  $\mathbb{E}[F(w_i)]-F(w^*)\geq \lambda \mathbb{E}[||w_i-w^*||_2^{\theta}]$ when we take expectations at both sides, namely
		\begin{equation}\label{eq1}
		\Delta_i\geq \lambda D_i^{\theta}. 
		\end{equation}
		Thus, we have 
		\begin{equation}\label{eq3}
		\Delta_{i}\leq cLD_{i-1}(\frac{1}{\sqrt{n_{i}}}+\frac{\sqrt{d\log(1/\delta)}}{\epsilon n_{i}})\overset{(\ref{eq1})}{\leq} cL (\frac{\Delta_{i-1}}{\lambda})^{\frac{1}{\theta}}(\frac{1}{\sqrt{n_{i}}}+\frac{\sqrt{d\log(1/\delta)}}{\epsilon n_{i}}),
		\end{equation}
		where the first inequality comes from Lemma \ref{lemma:4} and the second inequality uses (\ref{eq1}).
		Denote $E_i=\frac{c^{\theta}L^{\theta}}{\lambda}(\frac{1}{\sqrt{n_i}}+\frac{\sqrt{d\log(1/\delta)}}{\epsilon n_i})^{\theta}$.
		Then (\ref{eq3}) can be simplified as 
		\begin{equation}\label{eq2}
		\Delta_{i}\leq (\Delta_{i-1}E_i)^{\frac{1}{\theta}}.
		\end{equation}
		
		Notice that $n_{i}/n_{i-1}=2$, then $\frac{E_{i-1}}{E_{i}}\leq (\frac{n_{i}}{n_{i-1}})^{\theta}=2^{\theta}$, namely:
		\begin{equation}\label{eq4}
		E_{i}\geq 2^{-\theta} E_{i-1}.
		\end{equation}
		Then we can rearrange the above inequality as
		\begin{equation}\label{eq5}
		\frac{\Delta_{i}}{E_{i}^{\frac{1}{\theta-1}}}\leq \frac{(\Delta_{i-1}E_i)^{\frac{1}{\theta}}}{E_{i}^{\frac{1}{\theta-1}}} \leq2^{\frac{1}{\theta-1}}\left(\frac{\Delta_{i-1}}{E_{i-1}^{\frac{1}{\theta-1}}}\right)^{\frac{1}{\theta}},
		\end{equation}
		where the first inequality uses (\ref{eq2}) and the second inequality applies (\ref{eq4}).

		It can be verified that (\ref{eq5}) is equivalent to  \begin{equation*}
		\frac{\Delta_{i}}{2^{\frac{\theta}{(\theta-1)^2}}E_{i}^{\frac{1}{\theta-1}}}\leq\left(\frac{\Delta_{i-1}}{2^{\frac{\theta}{(\theta-1)^2}}E_{i-1}^{\frac{1}{\theta-1}}}\right)^{\frac{1}{\theta}}\leq \left(\frac{\Delta_1}{2^{\frac{\theta}{(\theta-1)^2}}E_1^{\frac{1}{\theta-1}}}\right)^{\frac{1}{\theta^{i-1}}}.
		\end{equation*}

		According to Lemma \ref{lemma:3}, $\Delta_1\leq (L^{\theta}\lambda^{-1})^{\frac{1}{\theta-1}}$. Also observe that $$E_1=\frac{c^{\theta}L^{\theta}}{\lambda}(\frac{1}{\sqrt{n_1}}+\frac{\sqrt{d\log(1/\delta)}}{\epsilon n_1})^{\theta}\geq \frac{c^{\theta}L^{\theta}}{\lambda}\frac{1}{(\sqrt{n_1})^{\theta}}\geq c^{\theta }\frac{L^{\theta}}{\lambda} \frac{1}{n^{\theta}}.$$ Let $c_1=c^{\frac{\theta}{\theta -1}}2^{\frac{\theta}{(\theta-1)^2}}$, then $ \frac{\Delta_1}{2^{\frac{\theta}{(\theta-1)^2}}E_1^{\frac{1}{\theta-1}}}\leq  \frac{n^{\frac{\theta}{\theta-1}} }{c_1 }$,
		which implies that for $k=\lfloor (\log_{\bar{\theta}}2)\cdot \log\log n\rfloor$, 
		\begin{equation*}
		\frac{\Delta_{k}}{2^{\frac{\theta}{(\theta-1)^2}}E_{k}^{\frac{1}{\theta-1}}}\leq  \left(\frac{n^{\frac{\theta}{\theta-1}} }{c_1}\right)^{\frac{1}{\theta^{k-1}}}.
		\end{equation*}
		Let $C_1=2^{\frac{\theta^3}{\theta -1}+\theta^2|\log c_1|}$. In the following we will prove that  $$\left(\frac{n^{\frac{\theta}{\theta-1}} }{c_1}\right)^{\frac{1}{\theta^{k-1}}}\leq C_1.$$
		Since $k+1\geq( \log_{\bar{\theta}} 2)\log\log n\geq ( \log_{\theta} 2)\log\log n$, it follows that
		$$(k-1)\log \theta +\log\log C_1\geq \log(\frac{\theta}{\theta-1}+|\log c_1|)+\log\log n,$$ which indicates $$(\frac{\theta}{\theta-1}+|\log c_1|)\log n\leq \theta^{k-1}\log C_1.$$ Thus we have $\frac{\theta}{\theta-1}\log n-\log c_1\leq \theta^{k-1}\log C_1$, which is equivalent to our object $\left(\frac{n^{\frac{\theta}{\theta-1}} }{c_1}\right)^{\frac{1}{\theta^{k-1}}}\leq C_1$.
		
		Now we know 
		\begin{equation*}
		\frac{\Delta_{k}}{2^{\frac{\theta^2}{(\theta-1)^2}}E_{k}^{\frac{1}{\theta-1}}}\leq  \left(\frac{n^{\frac{\theta}{\theta-1}} }{c_1}\right)^{\frac{1}{\theta^{k-1}}}\leq C_1,
		\end{equation*}
		which indicates that $\frac{\Delta_k}{E_{k}^{\frac{1}{\theta-1}}}\leq 2^{\frac{\theta}{(\theta-1)^2}} C_1=2^{\theta^2(\frac{\theta^2-\theta+1}{(\theta-1)^2}+|\log c_1|)}:=C.$

		As a result, we hold a solution with error:
		\begin{equation*}
		\begin{aligned}
		\mathbb{E}[F(w_k)]-F(w^*)\leq C E_k^{\frac{1}{\theta-1}}&=C
		\left(\frac{c^{\theta}L^{\theta}}{\lambda}\right)^{\frac{1}{\theta-1}}\left(\frac{1}{\sqrt{n_k}}+\frac{\sqrt{d\log(1/\delta)}}{\epsilon n_k}\right)^{\frac{\theta}{\theta-1}}\\
		&\leq  2^{\frac{3\theta}{2(\theta-1)}}\cdot C
		\left(\frac{c^{\theta}L^{\theta}}{\lambda}\right)^{\frac{1}{\theta-1}}\left(\frac{1}{n}+\frac{d\log(1/\delta)}{\epsilon^2 n^2}\right)^{\frac{\theta}{2(\theta-1)}}
		\end{aligned}
		\end{equation*}
		where we use the fact that $n_k=\frac{2^{k-1}}{(\log n)^{\log_{\bar{\theta}}2}}\geq \frac{1}{2} n$ and $(a+b)^2\leq 2(a^2+b^2)$.
	\end{proof}
	\begin{remark}
To perform valid Phased SGD (Subroutine of Iterated Phased-SGD) for $k$ times, it should satisfy $n_i=\frac{2^{i-1}n}{(\log n)^{\log_{\bar{\theta}} 2}}\geq 2$ for any $i\in [k]$. Otherwise, the Phased SGD cannot function properly to get the bound in Lemma \ref{lemma:4}. As a result, $n$ should be sufficiently large such that $\bar{\theta} \geq2^{\frac{\log\log n}{(\log n)-1}}$.
	\end{remark}
		\begin{proof}[{\bf Proof of Theorem \ref{thm:3}}]
	%	The term of $\Omega( (\frac{1}{\sqrt{n}})^\frac{\theta}{\theta-1})$ is the known lower bound on the excess population loss in the non-private case, which is shown by Theorem 2 in \citep{ramdas2012optimal}. In the following, we will show the bound of $\Omega\left( \left(\frac{\sqrt{d\log(1/\delta)}}{n\epsilon}\right)^{\frac{\theta}{\theta-1}}\right)$.
		
		Based on the fact that a lower bound on excess empirical risk implies nearly the same lower bound on the excess population risk \citep{bassily2019private}, here we consider the empirical risk, then we can use the boosting technique to the population loss. See \citep{bassily2019private} for details.

  Based on the definition of the loss function in (\ref{eq:7}), we can see that $f(w,x)$ is 2-Lipschiz in $||w||_2 \leq 1$, and it is $(\theta, \lambda)$-TNC with some constant $\lambda$ \citep{sridharan2010convex}. 
	
	For any dataset $S=\{x_1,\cdots,x_n\}$ with data point drawn from $x\in\{-\frac{1}{\sqrt{d}},\frac{1}{\sqrt{d}}\}^d$, and any $w\in \mathcal{W}$,
	we define the empirical risk function as the following, 
	\begin{equation*}
	\hat{F}(w;S)=\sum\limits_{i=1}^{n}\frac{1}{n}f(w,x_i)=-\langle w, \frac{1}{n}\sum\limits_{i=1}^{n} x_i \rangle +\frac{1}{\theta}||w||_2^{\theta}. 
	\end{equation*} 
		In the following, we first show that there is a point $w^{*}$ satisfying $ ||w^*||_2\leq 1$, s.t. $\nabla \hat{F}(w^*;S)=0$. To prove this, we first 
		take the derivative of $\hat{F}(w; S)$ and let it be $0$, so we get
			\begin{equation}\label{eq18}
		\nabla \hat{F}(w^*; S)=0\Leftrightarrow  ||w^*||_2^{\theta-2}\cdot w^*=\frac{\sum_{i=1}^{n}x_i}{n}
		\end{equation}
		That is $||w^*||_2^{\theta-1}=||\frac{\sum_{i=1}^{n}x_i}{ n}||_2\leq 1$, thus $w^*$ must satisfies $\|w^*\|_2\leq 1$ when $\theta>1$.

		In the following, we denote $\overline{\mathrm{Z}}=\frac{\sum_{i=1}^{n}x_i}{ n}$, then $||w^{*}||_2=||\overline{\mathrm{Z}}||_2^{\frac{1}{\theta-1}}$. Thus from  (\ref{eq18}) we can get  $w^{*}=\frac{\overline{\mathrm{Z}}}{||\overline{\mathrm{Z}}||_2^{\frac{\theta-2}{\theta-1}}}$. 
		Let $w_{priv}$ denote the output of the $(\epsilon,\delta)$-differentially private algorithm $\mathcal{A}$,  we will show that with probability at least $\frac{1}{3}$, 
			\begin{equation*}
		||w_{priv}-w^*||\geq \Omega\left( \left(\frac{\sqrt{d\log(1/\delta)}}{n\epsilon}\right)^{\frac{1}{\theta-1}}\right)
		\end{equation*}
		We prove it by showing that the following inequality leads to contradiction.
			\begin{equation}\label{eq19}
		\begin{aligned}
		||w_{priv}-w^*||\leq O\left( \left(\frac{\sqrt{d\log(1/\delta)}}{n\epsilon}\right)^{\frac{1}{\theta-1}}\right)
		\end{aligned}
		\end{equation}
If (\ref{eq19}) holds, then 
	\begin{equation}\label{eq20}
\begin{aligned}
||||\overline{\mathrm{Z}}||_2^{\frac{\theta-2}{\theta-1}}w_{priv}-\overline{\mathrm{Z}}||
\leq O\left( \left(\frac{\sqrt{d\log(1/\delta)}}{n\epsilon}\right)^{\frac{1}{\theta-1}}\cdot||\overline{\mathrm{Z}}||_2^{\frac{\theta-2}{\theta-1}}\right)
\end{aligned}
\end{equation}
Recall the following lemma.
\begin{lemma}[Lemma 5.1 in \citep{steinke2015between,bassily2014private}]\label{le8}
Let $n,d\in \mathbb{N}$, $\epsilon>0$ and $\delta=o(\frac{1}{n})$. There is a number $M=\Omega\left(\min\left(n,\frac{\sqrt{d\log(1/\delta)}}{\epsilon}\right)\right)$ such that for every $(\epsilon, \delta)$-differentially private algorithm $\mathcal{A}$, there is a dataset $S=\{x_1,\cdots,x_n\}\subseteq \{-\frac{1}{\sqrt{d}},\frac{1}{\sqrt{d}}\}^d$ with $||\sum_{i=1}^{n}x_i||_2\in[M-1,M+1]$ such that w.p. $\frac{1}{3}$, we have 
	\begin{equation*}
||\mathcal{A}(S)-\frac{1}{n}\sum_{i=1}^nx_i||_2=\Omega\left(\min\left(1,\frac{\sqrt{d\log(1/\delta)}}{\epsilon n}\right)\right)
\end{equation*}
\end{lemma}
For the sake of contradiction, we consider such $S$ described in the above lemma, with probability more than $\frac{2}{3}$, (\ref{eq19}) holds. Let $\mathcal{\tilde{A}}$ be an $(\epsilon, \delta)$-differentially private algorithm that first runs $\mathcal{A}$ on the data and then outputs $||\overline{\mathrm{Z}}||_2^{\frac{\theta-2}{\theta-1}}w_{priv}$, and let $n$ be sufficiently large that $n\geq \frac{\sqrt{d\log(1/\delta)}}{\epsilon}$.

 Then we have $||\overline{\mathrm{Z}}||_2^{\frac{\theta-2}{\theta-1}}=\Theta\left(\left( \frac{\sqrt{d\log(1/\delta)}}{n\epsilon} \right)^{\frac{\theta-2}{\theta-1}}\right)$, and (\ref{eq20}) will become
 	\begin{equation*}
 \begin{aligned}
 ||||\overline{\mathrm{Z}}||_2^{\frac{\theta-2}{\theta-1}}w_{priv}-\overline{\mathrm{Z}}||=||\tilde{\mathcal{A}}-\overline{\mathrm{Z}}||
 \leq O\left( \frac{\sqrt{d\log(1/\delta)}}{n\epsilon}\right)
 \end{aligned}
 \end{equation*}
 which contradicts to Lemma \ref{le8}.
 Thus 	
 \begin{equation}
 \hat{F}(w_{priv},S)-\hat{F}(w^{*},S)  \geq\Omega\left( \left(\frac{\sqrt{d\log(1/\delta)}}{n\epsilon}\right)^{\frac{\theta}{\theta-1}}\right)
 \end{equation}
 By the boosting technique in \citep{bassily2019private}, we have with probability at least $\frac{1}{3}$, 
 \begin{equation*}
 F(w_{priv})-\min_{\|w\|_2\leq 1}F(w) \geq\Omega\left( \left(\frac{\sqrt{d\log(1/\delta)}}{n\epsilon}\right)^{\frac{\theta}{\theta-1}}\right).
 \end{equation*}
	\end{proof}
		\begin{proof}[{\bf Proof of Theorem \ref{thm:4}}]
	The proof of Theorem \ref{thm:4} is almost the same as the proof of Theorem \ref{thm:3}. Instead of using Lemma \ref{le8} we use the following lemma: 
	\begin{lemma}[Lemma 5.1 in \citep{bassily2014private}]
	    Let $n,d\in \mathbb{N}$, $\epsilon>0$ such that $n\geq \Omega(\frac{d}{\epsilon})$. There is a number $M=\Omega\left(\min\left(n,\frac{d}{\epsilon}\right)\right)$ such that for every $(\epsilon, \delta)$-differentially private algorithm $\mathcal{A}$, there is a dataset $S=\{x_1,\cdots,x_n\}\subseteq \{-\frac{1}{\sqrt{d}},\frac{1}{\sqrt{d}}\}^d$ with $||\sum_{i=1}^{n}x_i||_2\in[M-1,M+1]$ such that w.p. $\frac{1}{3}$, we have 
	\begin{equation}
||\mathcal{A}(S)-\frac{1}{n}\sum_{i=1}^nx_i||_2=\Omega\left(\min\left(1,\frac{d}{\epsilon n}\right)\right).
\end{equation}
	\end{lemma}
	\end{proof}

\begin{proof}[{\bf Proof of Theorem \ref{thm:6}}]
For simplicity, here we only focus on $(\epsilon,\delta)$-DP. It is almost the same for $\epsilon$-DP. 

In the first step we perform Algorithm \ref{alg:1}, which is $(\epsilon,\delta)$-DP. Thus,  it is sufficient to show that  Algorithm \ref{alg:5} is also $(\epsilon,\delta)$-DP, {\em i.e.,} each epoch is $(\epsilon,\delta)$-DP. To prove this, we first revoke the stability of One -Pass Projected SGD for strongly convex loss functions, which is given by \citep{hardt2015train}. 

\begin{lemma}\label{le1}[Theorem 3.9 in \citep{hardt2015train}]
	Assume the loss function $f(\cdot, x)$ is $\lambda$-strongly convex and $\beta$-smooth with respect to $w\in \mathcal{W}$ for all $x$. Let $S_i$ and $S_i^{'}$ be two samples of size $n_i$ differing in only a single element. Denote $w_i^t$ and ${w'}_i^t$ as the outputs of the projected stochastic gradient method (\ref{eq6}) on datasets $S_i$ and $S_i^{'}$ respectively, then if $\eta\leq \frac{1}{\beta}$  we have
	\begin{equation}
	||w_i^t-{w'}_i^t||\leq \frac{2L^2}{\lambda n_i}
	\end{equation}
\end{lemma}

Recall that in each epoch we perform projected gradient descent for $n_i$ steps using $n_i$ samples, according to Lemma \ref{le1}, we can bound the sensitivity of $w_i^t$ for each $t$ and we have 
	$
	||w_i^t-{w'}_i^t||\leq \frac{2L^2}{\lambda n_i}
	$
	for all $t$, where $w_i^t$ and ${w'}_i^t$ correspond to the solution of two neighboring dataset $S_i$ and $S^{'}_i$ that differs in one sample. 
	
	Thus, the sensitivity of 	$\overline{w}_i=\frac{1}{n_i}\sum \limits_{t=1}^{n_i} w_i^t$ is also $\frac{2L^2}{\lambda n_i}$. By the Gaussian mechanism, adding Gaussian noise with $\sigma_i=\frac{8L^2\sqrt{\log(1/\delta)}}{n_i\lambda\epsilon}$ will preserve $(\epsilon,\delta)$-DP.
\end{proof}
\begin{proof}[{\bf Proof of Theorem \ref{thm:7}}]
For convenience here we only focus on $(\epsilon,\delta)$-DP, the proof is almost the same as for $\epsilon$-DP. 

Since $F(\cdot)$ is $\lambda$-strongly convex, it satisfies $(2, \frac{\lambda}{2})$-TNC. Thus, by Theorem \ref{thm:2} we have
\begin{align}\label{eq11}
\mathbb{E}[F(\hat{w})]-F(w^*)\leq \frac{c^2L^2}{\lambda/2}\left(\frac{1}{n/2}+\frac{d\log(1/\delta)}{\epsilon^2 (n/2)^2}\right)&\leq\frac{c_1^2L^2}{\lambda}\left(\frac{1}{n}+\frac{d\log(1/\delta)}{\epsilon^2n^2}\right) \notag \\
&\leq \frac{c_1^2L^2}{\lambda}\left(\frac{1}{\kappa^{\tau}}+\frac{d\log(1/\delta)}{\epsilon^2\kappa^{2\tau}}\right),
\end{align}
where $c$ and $c_1$ are universal constants and the last inequality is due to the condition of $n\geq \kappa^\tau$. 

Now we proceed to analyze the solution returned by Epoch-DP-SGD (Algorithm \ref{alg:5}). The following lemma shows how the excess population risk decreases in each epoch.
\begin{lemma}[Lemma 1 in \citep{zhang2019stochastic}]\label{le5}
 Assume $f(\cdot, x)$ is non-negative and $\beta$-smooth for all $x$ and $F(\cdot)$ is convex.	Apply $n_i$ iterations of  (\ref{eq6}), {\em i.e.,} $w_i^{t+1}=\prod_{\mathcal{W}}(w_i^{t}-\eta_i\nabla_{w}f(w_{i}^{t},x_i^t))$ with $\eta_i<1/(2\beta)$. Then for any $w\in \mathcal{W}$, we have
	\begin{equation*}
	\mathbb{E}[F(\overline{w}_i)]-F(w)\leq \frac{1}{2\eta_i n_i(1-2\eta_i \beta)} \mathbb{E}[||w_i^1-w||^2]+\frac{2\eta_i\beta}{(1-2\eta_i\beta)}F(w),
	\end{equation*}
	where $\overline{w}_i=\frac{1}{n_i}\sum \limits_{t=1}^{n_i} w_i^t$.
\end{lemma}
Since $f(\cdot, x)$ is $\beta$-smooth for all $x$, we have 
\begin{align*}
f(w_i)-f(\overline{w}_i)\leq& \langle \nabla f(\overline{w}_i), w_i-\overline{w}_i\rangle +\frac{\beta}{2}||w_i-\overline{w}_i||_2^2\\=&\langle \nabla f(\overline{w}_i), \xi_i\rangle +\frac{\beta}{2}||\xi_i||_2^2
\end{align*}
Take expectations on both sides w.r.t the data and $\xi_i$ we  get 
\begin{equation*}
\mathbb{E}[F(w_i)]-F(\overline{w}_i)\leq\frac{\beta}{2}\mathbb{E}[||\xi_i||_2^2]=\frac{d\beta\sigma_i^2}{2}=\frac{32dL^4\beta\log(1/\delta)}{n_i^2\epsilon^2\lambda^2}.
\end{equation*}
Combining with Lemma \ref{le5}, we have 
\begin{equation}\label{eq9}
\begin{aligned}
&\mathbb{E}[F(w_i)]-F(w^{*})\\=&\mathbb{E}[F(w_i)]-F(\overline{w}_i)+F(\overline{w}_i)-F(w^{*})\\\leq&\frac{32dL^4\beta\log(1/\delta)}{n_i^2\epsilon^2\lambda^2}+ \frac{1}{2\eta_i n_i(1-2\eta_i \beta)} \mathbb{E}[||w_i^1-w^{*}||^2]+\frac{2\eta_i\beta}{(1-2\eta_i\beta)}F(w^{*})
\end{aligned}
\end{equation}
Based on the above result, we establish the following result of excess population risk of each epoch in Epoch-DP-SGD (Algorithm \ref{alg:5}).
\begin{lemma}\label{le6}
 For any epoch $e$ in Epoch-DP-SGD (Algorithm \ref{alg:5}), we have 
	\begin{equation*}
	\begin{aligned}
	\mathbb{E}[F(w_e)]-F(w^*)\leq &\left(\frac{32 dL^4\beta\log(1/\delta)}{n_e^2\epsilon^2\lambda^2}+\frac{2^{2\tau+3}\cdot \kappa \cdot F(w^{*})}{n_e}\right) \cdot \sum\limits_{i=1}^{e}\frac{1}{2^{2(i-1)(\tau -1)}}\\
	+&\frac{c_1^2 L^2}{\lambda}\left(\frac{2^{2\tau^2+\tau}}{n_e^{\tau}}+\frac{2^{4\tau^2+4\tau  }\cdot d\log(1/\delta)}{n_e^{2\tau}\cdot \epsilon^2}\right)
	\end{aligned}
	\end{equation*}
\end{lemma}
\begin{proof}[{\bf Proof of Lemma \ref{le6}}]
	We will prove the lemma by induction on $e$. 
	
	Note that by iteration rules in our algorithm, $w_1^1=\hat{w}$, $ w_{e+1}^1=w_e$, also, by the algorithm setting, we have for any epoch $e$, 
	\begin{equation}\label{eq13}
	\eta_e\beta\leq \eta_1\beta= \frac{1}{4}.
	\end{equation} 
	\begin{equation}\label{eq14}
	\eta_{e}n_e=\eta_1n_1=2^{2\tau+3}\kappa\cdot \frac{1}{4\beta}.
	\end{equation} 
	
When $e=1$, from (\ref{eq9}), we have
	\begin{equation*}
	\begin{aligned}
	\mathbb{E}[F(w_1)]-F(w^{*})\leq&\frac{32dL^4\beta\log(1/\delta)}{n_1^2\epsilon^2\lambda^2}+ \frac{1}{2\eta_1 n_1(1-2\eta_1 \beta)} \mathbb{E}[||w_1^1-w^{*}||^2]+\frac{2\eta_1\beta}{(1-2\eta_1\beta)}F(w^{*})\\
	\overset{(\ref{eq14})}{	\leq}&\frac{32dL^4\beta\log(1/\delta)}{n_1^2\epsilon^2\lambda^2}+ \frac{\lambda}{2^{2\tau+1}} \mathbb{E}[||w_1^1-w^{*}||^2]+4\eta_1\beta F(w^{*})\\
{\leq}&\frac{32dL^4\beta\log(1/\delta)}{n_1^2\epsilon^2\lambda^2}+ \frac{\lambda}{2^{2\tau+1}} \cdot \frac{2}{\lambda}\mathbb{E}[F(w_1^1)-F(w^{*})]+4\eta_1\beta F(w^{*})\\
	\overset{(\ref{eq11})}{\leq}&\frac{32dL^4\beta\log(1/\delta)}{n_1^2\epsilon^2\lambda^2}+ \frac{1}{2^{2\tau}} \frac{c_1^2L^2}{\lambda}\left(\frac{1}{\kappa^{\tau}}+\frac{d\log(1/\delta)}{\epsilon^2{\kappa}^{2\tau}}\right)+4\eta_1\beta F(w^{*})\\
	\overset{(\ref{eq14})}{\leq} &\frac{32dL^4\beta\log(1/\delta)}{n_1^2\epsilon^2\lambda^2}+ \frac{1}{2^{2\tau}} \frac{c_1^2L^2}{\lambda}\left(\frac{2^{2\tau^2+3\tau}}{n_1^{\tau}}+\frac{2^{4\tau^2+6\tau}d\log(1/\delta)}{\epsilon^2n_1^{2\tau}}\right)+\frac{	2^{2\tau+3} \cdot \kappa F(w^{*})}{n_1}
	\\
	\leq &\frac{32dL^4\beta\log(1/\delta)}{n_1^2\epsilon^2\lambda^2}+ \frac{c_1^2L^2}{\lambda}\left(\frac{2^{2\tau^2+\tau}}{n_1^{\tau}}+\frac{2^{4\tau^2+4\tau}d\log(1/\delta)}{\epsilon^2n_1^{2\tau}}\right)+\frac{	2^{2\tau+3} \cdot \kappa F(w^{*})}{n_1}.
	\end{aligned}
	\end{equation*}
	Thus the lemma holds for $e=1$. Now we assume the lemma is true for some $e\geq 1$, then for $e+1$, 
	\begin{equation*}
	\begin{aligned}
&\mathbb{E}[F(w_{e+1})]-F(w^*)\\
\overset{(\ref{eq9})}{\leq}& \frac{32dL^4\beta\log(1/\delta)}{n_{e+1}^2\epsilon^2\lambda^2}+ \frac{1}{2\eta_{e+1} n_{e+1}(1-2\eta_{e+1} \beta)} \mathbb{E}[||w_{e+1}^1-w^{*}||^2]+\frac{2\eta_{e+1}\beta}{(1-2\eta_{e+1}\beta)}F(w^{*})\\
	\overset{(\ref{eq13})}{\leq} 	& \frac{32dL^4\beta\log(1/\delta)}{n_{e+1}^2\epsilon^2\lambda^2}+\frac{1}{\eta_{e+1} n_{e+1}} \mathbb{E}[||w_{e+1}^1-w^{*}||^2]+4\eta_{e+1}\beta F(w^{*})\\
	\overset{ (\ref{eq14})}{\leq}	& \frac{32dL^4\beta\log(1/\delta)}{n_{e+1}^2\epsilon^2\lambda^2}+\frac{\lambda}{2^{2\tau+1}}\cdot\frac{2}{\lambda} \mathbb{E}[F(w_e)-F(w^{*})]+\frac{\kappa \cdot 2^{2\tau+3}}{n_{e+1}} F(w^{*})\\
	{\leq}	& \frac{32dL^4\beta\log(1/\delta)}{n_{e+1}^2\epsilon^2\lambda^2}+\frac{1}{2^{2\tau}} \left(\frac{32dL^4\beta\log(1/\delta)}{n_e^2\epsilon^2\lambda^2}+\frac{2^{2\tau+3}\cdot \kappa \cdot F(w^{*})}{n_e}\right) \cdot \sum\limits_{i=1}^{e}\frac{1}{2^{2(i-1)(\tau -1)}}\\
	+&\frac{1}{2^{2\tau}}\frac{c_1^2 L^2}{\lambda}\left(\frac{2^{2\tau^2+\tau}}{n_e^{\tau}}+\frac{2^{4\tau^2+4\tau  }\cdot d\log(1/\delta)}{n_e^{2\tau}\cdot \epsilon^2}\right)+\frac{\kappa \cdot 2^{2\tau+3}}{n_{e+1}} F(w^{*})\\
	<	&  \frac{32dL^4\beta\log(1/\delta)}{n_{e+1}^2\epsilon^2\lambda^2}\left(1+\frac{1}{2^{2\tau-2}}\cdot \sum\limits_{i=1}^{e}\frac{1}{2^{2(i-1)(\tau -1)}}\right)+\frac{c_1^2 L^2}{\lambda}\left(\frac{2^{2\tau^2+\tau}}{n_{e+1}^{\tau}}+\frac{2^{4\tau^2+4\tau  }\cdot d\log(1/\delta)}{n_{e+1}^{2\tau}\cdot \epsilon^2}\right)\\
	+&\frac{\kappa \cdot 2^{2\tau+3}}{n_{e+1}} F(w^{*})\left(1+\frac{1}{2^{2\tau-1}}\sum\limits_{i=1}^{e}\frac{1}{2^{2(i-1)(\tau -1)}}\right)\\
	< &\left(\frac{32dL^4\beta\log(1/\delta)}{n_{e+1}^2\epsilon^2\lambda^2}+\frac{2^{2\tau+3}\cdot \kappa \cdot F(w^{*})}{n_{e+1}}\right) \cdot \sum\limits_{i=1}^{{e+1}}\frac{1}{2^{2(i-1)(\tau -1)}}\\
	+&\frac{c_1^2 L^2}{\lambda}\left(\frac{2^{2\tau^2+\tau}}{n_{e+1}^{\tau}}+\frac{2^{4\tau^2+4\tau  }\cdot d\log(1/\delta)}{n_{e+1}^{2\tau}\cdot \epsilon^2}\right).
	\end{aligned}
	\end{equation*}
	Thus the lemma holds for $e+1$ which completes the proof. 
\end{proof}
Now we go back to our proof. The number of epochs made is given by the largest $e$ which satisfies $\sum\limits_{i=1}^{e} n_i\leq \frac{n}{2}$, i.e.,
\begin{equation*}
\sum\limits_{i=1}^{e} n_i=n_1(1+2+\cdots+2^{e-1})=n_1(2^e-1)\leq \frac{n}{2}
\end{equation*}
which means the largest value is
$E=\lfloor\log_2(\frac{n}{2n_1}+1)\rfloor$
and the final solution is $\tilde{w}=w_E$.

From Lemma \ref{le6}, we have 
\begin{equation*}
\begin{aligned}
\mathbb{E}[F(w_E)]-F(w^*)\leq &\left(\frac{32L^4 \beta d\log(1/\delta)}{\lambda^2 n_E^2\cdot\epsilon^2}+\frac{2^{2\tau+3}\cdot \kappa  F(w^{*})}{n_E}\right) \cdot \sum\limits_{i=1}^{E}\frac{1}{2^{2(i-1)(\tau -1)}}\\
+&\frac{c_1^2 L^2}{\lambda}\left(\frac{2^{2\tau^2+\tau}}{n_E^{\tau}}+\frac{2^{4\tau^2+4\tau  }\cdot d\log(1/\delta)}{n_E^{2\tau}\cdot \epsilon^2}\right)
\\\leq &\left(\frac{32L^4 \beta d\log(1/\delta)}{\lambda^2 n_E^2\epsilon^2}+\frac{2^{2\tau+3}\cdot \kappa  F(w^{*})}{n_E}\right) \cdot \frac{2^{2\tau-2}}{2^{2\tau-2}-1}\\
+&\frac{c_1^2 L^2}{\lambda}\left(\frac{2^{2\tau^2+\tau}}{n_E^{\tau}}+\frac{2^{4\tau^2+4\tau  }\cdot d\log(1/\delta)}{n_E^{2\tau}\cdot \epsilon^2}\right)
\\\leq &\left(\frac{2^{2\tau +9}L^4 \beta d\log(1/\delta)}{\lambda^2 n^2\epsilon^2}+\frac{2^{4\tau+4}\cdot \kappa  F(w^{*})}{n}\right) \cdot \frac{1}{2^{2\tau-2}-1}\\
+&\frac{c_1^2 L^2}{\lambda}\left(\frac{2^{2\tau^2+4\tau}}{n^{\tau}}+\frac{2^{4\tau^2+10\tau  }\cdot d\log(1/\delta)}{n^{2\tau}\cdot \epsilon^2}\right)
\\=&O\left( \frac{L^4 \beta d\log(1/\delta)}{\lambda^2 n^2\epsilon^2}+\frac{4^{\tau}\cdot \kappa  F(w^{*})}{n} +\frac{c_1^2 L^2}{\lambda}\left(\frac{2^{2\tau^2+4\tau}}{n^{\tau}}+\frac{2^{4\tau^2+10\tau  }\cdot d\log(1/\delta)}{n^{2\tau}\cdot \epsilon^2}\right)  \right)
\end{aligned}
\end{equation*}
where the last step is due to the fact that 
$n_{E}=n_1 2^{E-1}\geq \frac{n_1}{4}(\frac{n}{2n_1}+1)\geq \frac{n}{8}$.
\end{proof}
%\vspace{-0.1in}
\end{document}